\documentclass{article}

\usepackage{iclr2023_conference,times}
\iclrfinalcopy%

\usepackage[utf8]{inputenc}
\usepackage[T1]{fontenc}
\usepackage{hyperref}
\usepackage{url}
\usepackage{booktabs}
\usepackage{amsfonts}
\usepackage{nicefrac}
\usepackage{microtype}
\usepackage{xcolor}
\usepackage{amsmath}
\usepackage{amssymb}
\usepackage{amsmath}
\usepackage{amsthm}
\usepackage{mathtools}
\usepackage{bm}
\usepackage{graphicx}
\usepackage{subcaption}
\usepackage{cleveref}
\usepackage{pifont}
\usepackage{tabularx}
\usepackage{wrapfig}
\usepackage{thmtools}
\usepackage{thm-restate}
\usepackage{stmaryrd}
\usepackage{enumerate}
\usepackage{cancel}
\usepackage[normalem]{ulem}
\usepackage{marginnote}
\usepackage{wrapfig}
\usepackage{lipsum}
\usepackage[symbol]{footmisc}
\usepackage{fancyvrb}
\usepackage{dsfont}
\usepackage{tikz}
\usetikzlibrary{arrows}
\usetikzlibrary{bayesnet}

\theoremstyle{plain}

\newtheorem{assumption}{Assumption}
\newtheorem{definition}{Definition}
\DeclareMathOperator{\E}{\mathbb{E}}
\DeclareMathOperator{\InfoNCE}{\mathcal{L}_{\text{InfoNCE}}}
\DeclareMathOperator{\SymInfoNCE}{\mathcal{L}_{\text{SymInfoNCE}}}
\DeclareMathOperator{\AME}{\mathcal{L}_{\text{AlignMaxEnt}}}
\DeclareMathOperator{\SymAME}{\mathcal{L}_{\text{SymAlignMaxEnt}}}
\DeclareMathOperator{\simi}{sim}
\DeclareBoldMathCommand{\f}{\mathrm{f}}
\DeclareBoldMathCommand{\g}{\mathrm{g}}
\DeclareBoldMathCommand{\h}{\mathrm{h}}
\DeclareBoldMathCommand{\x}{\mathrm{x}}
\DeclareBoldMathCommand{\X}{X}
\DeclareBoldMathCommand{\y}{y}
\DeclareBoldMathCommand{\Y}{Y}
\DeclareBoldMathCommand{\z}{\mathrm{z}}
\DeclareBoldMathCommand{\Z}{Z}
\DeclareBoldMathCommand{\s}{\mathrm{s}}
\DeclareBoldMathCommand{\c}{\mathrm{c}}
\DeclareBoldMathCommand{\n}{\mathrm{n}}
\DeclareBoldMathCommand{\m}{\mathrm{m}}
\DeclareBoldMathCommand{\a}{\mathrm{a}}
\DeclareBoldMathCommand{\tc}{\mathrm{\tilde{c}}}
\DeclareBoldMathCommand{\ts}{\mathrm{\tilde{s}}}
\DeclareBoldMathCommand{\tm}{\mathrm{\tilde{m}}}
\DeclareBoldMathCommand{\tz}{\mathrm{\tilde{z}}}

\newcommand{\Indep}{\mathop{\perp\!\!\!\perp}\nolimits} 
\newcommand{\cmark}{\ding{51}}%
\newcommand{\xmark}{\ding{55}}%
\newcommand\labelAndRemember[2]
  {\expandafter\gdef\csname labeled:#1\endcsname{#2}\label{#1}#2}
\newcommand\recallLabel[1]
   {\csname labeled:#1\endcsname\tag{\ref{#1}}}

\hypersetup{%
    colorlinks,
    linkcolor={red!50!black},
    citecolor={blue!50!black},
    urlcolor={blue!80!black}
}

\title{Identifiability Results for\\ Multimodal Contrastive Learning}

\author{%
\footnotesize{Imant Daunhawer$^{1,\dagger}$, Alice Bizeul$^{1,2}$, Emanuele Palumbo$^{1,2}$, \textbf{Alexander Marx}$^{1,2,*}$ \& \textbf{Julia E. Vogt}$^{1,*}$} \vspace{0.25em}\\
\ \footnotesize{$^1$\,Department of Computer Science, ETH Zurich} \\
\ \footnotesize{$^2$\,ETH AI Center, ETH Zurich} \\ 
}

\begin{document}

\maketitle

\begin{abstract}
  
Contrastive learning is a cornerstone underlying recent progress in multi-view
and multimodal learning, e.g., in representation learning with image/caption
pairs. While its effectiveness is not yet fully understood, a line of recent
work reveals that contrastive learning can invert the data generating process
and recover ground truth latent factors shared between views. In this work, we
present new identifiability results for multimodal contrastive learning,
showing that it is possible to recover shared factors in a more general setup
than the multi-view setting studied previously. Specifically, we distinguish
between the multi-view setting with one generative mechanism (e.g., multiple
cameras of the same type) and the multimodal setting that is characterized by
\emph{distinct} mechanisms (e.g., cameras and microphones). Our work
generalizes previous identifiability results by redefining the generative
process in terms of distinct mechanisms with modality-specific latent
variables. We prove that contrastive learning can block-identify latent factors
shared between modalities, even when there are nontrivial dependencies between
factors. We empirically verify our identifiability results with numerical
simulations and corroborate our findings on a complex multimodal dataset of
image/text pairs. Zooming out, our work provides a theoretical basis for
multimodal representation learning and explains in which settings multimodal
contrastive learning can be effective in practice.

\end{abstract}

\section{Introduction}

\renewcommand{\thefootnote}{\fnsymbol{footnote}}
\footnotetext[1]{Joint authorship. $^\dagger$Correspondence to: \texttt{dimant@ethz.ch}.}
\renewcommand{\thefootnote}{\arabic{footnote}}

Multimodal representation learning is an emerging field whose growth is fueled
by recent developments in weakly-supervised learning algorithms and by the
collection of suitable multimodal datasets. Multimodal data is characterized by
the \textit{co-occurence} of observations from two or more dependent data
sources, such as paired images and captions
\citep[e.g.,][]{Salakhutdinov2009,Shi2019,Radford2021}, and more generally,
multimodal observations are comprised of aligned measurements from different
types of sensors \citep{Baltrusaitis2019}. Co-occurrence is a form of
\textit{weak supervision} \citep{Shu2020,Locatello2020,Chen2020_weak}, in that
paired observations can be viewed as proxies (i.e., weak labels) for a shared
but unobserved ground truth factor.  Among suitable representation learning
methods for weakly supervised data, \textit{contrastive~learning}
\citep{Gutmann2010,Oord2018} stands out because it is designed to leverage
co-occurring observations from different views.  In practice, contrastive
learning achieves promising results for multi-view and multimodal learning---a
prominent example is the contribution of \texttt{CLIP} \citep{Radford2021} to
the groundbreaking advancements in text-to-image generation
\citep{Ramesh2021,Ramesh2022,Rombach2022,Saharia2022}.

Despite its empirical success, it is not sufficiently well understood what
explains the effectiveness of contrastive learning in practice. Recent works
attribute its effectiveness to the recovery of shared latent factors from the
underlying causal graph \citep{Gresele2019,Zimmermann2021,Kuegelgen2021}. From
the perspective of multi-view independent component analysis, it was shown that
contrastive learning can invert a nonlinear mixing function (i.e., a nonlinear
generative process) that is applied to a latent variable with mutually
independent components \citep{Gresele2019,Zimmermann2021}.  More recently,
\citet{Kuegelgen2021} show that contrastive learning can recover shared factors
up to block-wise indeterminacies, even if there are nontrivial causal and
statistical dependencies between latent components.  Collectively, these
results suggest that contrastive learning can identify parts of an unknown data
generating process from pairs of observations alone---even from
high-dimensional multi-view observations with nontrivial dependencies.  In our
work, we investigate the identifiability of shared latent factors in the
\emph{multimodal} setting.

We consider a generative process with modality-specific mixing functions and
modality-specific latent variables.  Our design is motivated by the inherent
heterogeneity of multimodal data, which follows naturally when observations are
generated by different types of sensors \citep{Baltrusaitis2019}. For example,
an agent can perceive its environment through distinct sensory modalities, such
as cameras sensing light or microphones detecting sound waves.  To model
information that is shared between modalities, we take inspiration from the
multi-view setting~\citep{Kuegelgen2021} and allow for nontrivial dependencies
between latent variables.  However, previous work only considers observations
of the same data type and assumes that the same input leads to the same output
across views.  In this work, we introduce a model with \textit{distinct
generative mechanisms}, each of which can exhibit a significant degree of
modality-specific variation.  This distinction renders the multimodal setting
more general compared to the multi-view setting considered by previous work.

In a nutshell, our work is concerned with \emph{identifiability for multimodal
representation learning} and focuses on \emph{contrastive~learning} as a
particular algorithm for which we derive identifiability results. In
\Cref{sec:preliminaries}, we cover relevant background on both topics,
identifiability and contrastive learning.  We then formalize the multimodal
generative process as a latent variable model (\Cref{sec:generative_process})
and prove that contrastive learning can block-identify latent factors shared
between modalities (\Cref{sec:identifiability_results}). We empirically verify
the identifiability results with fully controlled numerical simulations
(\Cref{subsec:numerical_experiment}) and corroborate our findings on a complex
multimodal dataset of image/text pairs (\Cref{subsec:imagetext_experiment}).
Finally, we contextualize related literature~(\Cref{sec:related_work}) and
discuss potential limitations and opportunities for future work
(\Cref{sec:discussion}).

\section{Preliminaries}
\label{sec:preliminaries}

\subsection{Identifiability}
\label{subsec:preliminaries_identifiability}

Identifiability lies at the heart of many problems in the fields of independent
component analysis~(ICA), causal discovery, and inverse problems, among others
\citep{Lehmann2006}. From the perspective of ICA, we consider the relation $\x
= \f(\z)$, where an observation $\x$ is generated from a mixing function $\f$
that is applied to a latent variable $\z$. The goal of ICA is to invert the
mixing function in order to recover the latent variable \emph{from observations
alone}. In many settings, full identifiability is impossible and certain
ambiguities might be acceptable. For example, identifiability might hold for a
subset of components (i.e., partial identifiability). Typical ambiguities
include permutation and element-wise transformations (i.e., component-wise
indeterminacy), or identifiability up to groups of latent variables (i.e.,
block-wise indeterminacy). In the general case, when $\f$ is a nonlinear
function, a landmark negative result states that the recovery of the latent
variable given i.i.d.~observations is fundamentally impossible
\citep{Hyvaerinen1999}. However, a recent line of pioneering works provides
identifiability results for the difficult nonlinear case under additional
assumptions, such as auxiliary variables
\citep{Hyvaerinen2017,Hyvaerinen2019,Khemakhem2020} or multiple views
\citep{Gresele2019,Locatello2020,Zimmermann2021}. 

Most relevant to our investigation are previous works related to
\emph{multi-view nonlinear ICA}
\citep{Gresele2019,Lyu2020,Locatello2020,Kuegelgen2021,Lyu2022}. Generally,
this line of work considers the following generative process: 
\begin{equation}
\label{eq:multiview_ica} \z \sim p_{\z}, \quad \x_1 = \f_1(\z), \quad \x_2 =
\f_2(\z), \end{equation} 
where a latent variable, or a subset of its components, is shared between
\emph{pairs} of observations $(\x_1, \x_2) \sim p_{\x_1, \x_2}$, where the two
views $\x_1$ and $\x_2$ are generated by two nonlinear mixing functions, $\f_1$
and $\f_2$ respectively.  Intuitively, a second view can resolve ambiguity
introduced by the nonlinear mixing, if both views contain a shared signal but
are otherwise sufficiently distinct \citep{Gresele2019}.  Previous works differ
in their assumptions on the mixing functions and dependence relations between
latent components.  The majority of previous work considers mutually
independent latent components \citep{Song2014,Gresele2019,Locatello2020} or
independent groups of shared and view-specific components
\citep{Lyu2020,Lyu2022}. Moreover, some of these works
\citep{Song2014,Gresele2019,Lyu2020,Lyu2022} consider view-specific%
\footnote{\label{footnote:view_specific}%
  Note that we define \emph{modality-specific} functions similar to
  the way \citet{Gresele2019}, \citet{Lyu2020}, and \citet{Lyu2022} define \emph{view-specific}
  functions. To clarify the distinction, we 
  generally assume that observations from different modalities are generated
  by distinct mechanisms $\f_1 \not = \f_2$ with modality-specific latent variables, and we treat the multi-view setting as a special
  case, where $\f_1 = \f_2$ without view-specific~latents.
}
mixing functions.  Venturing beyond the strict assumption of independent
(groups~of) components, \citet{Kuegelgen2021} consider additional causal and
statistical dependencies between latents and show that the subset of shared
components can be identified up to a block-wise indeterminacy.  Our work
considers heterogeneous modalities with causal and statistical dependencies
between latents. We prove that shared factors can be block-identified in a
novel setting with modality-specific mixing functions and modality-specific
latent variables.

\subsection{Contrastive Learning}

Contrastive learning \citep{Gutmann2010,Oord2018} is a self-supervised
representation learning method that leverages weak supervision in the form of
paired observations. On a high level, the method learns to distinguish
``positive'' pairs of encodings sampled from the joint distribution, against
``negative'' pairs sampled from the product of marginals. The popular InfoNCE
objective \citep{Oord2018} is defined as follows:
\begin{equation}\label{eq:info-nce-def}
  \InfoNCE(\g_1, \g_2) =
  \E_{\{\x_1^{i},\x_2^{i}\}_{i=1}^{K} \sim p_{\x_1, \x_2}} \left[
  - \sum_{i=1}^{K} \log \frac{\exp\{\simi(\g_1(\x_1^{i}), \g_2(\x_2^{i}))/\tau\}}{\sum_{j=1}^{K}\exp\{\simi(\g_1(\x_1^{i}), \g_2(\x_2^{j}))/\tau\}} \right] \;,
\end{equation}
where $\g_1$ and $\g_2$ are encoders for the first and second view, $\x_1$ and
$\x_2$ respectively.  It is common to use a single encoder $\g_1 = \g_2$ when
$\x_2$ is an augmented version of $\x_1$ or when two augmentations are sampled
from the same distribution to transform $\x_1$ and $\x_2$ respectively
\citep[e.g.,][]{Chen2020_simclr}. The set of hyperparameters consists of the
temperature $\tau$, a similarity metric $\text{sim}(\cdot,\cdot)$, and an
integer $K$ that controls the number of negative pairs ($K-1$) used for
contrasting. The objective has an information-theoretic interpretation as a
variational lower bound on the mutual information $I(\g_1(\x_1); \g_2(\x_2))$
\citep{Oord2018,Poole2019} and it can also be interpreted as the alignment of
positive pairs (numerator) with additional entropy regularization
(denominator), where the regularizer disincentivizes a degenerate solution in
which both encoders map to a constant \citep{Wang2020}.  Formally, when
instantiating the $\InfoNCE$ objective with $\tau = 1$ and $\simi(a, b) = -(a - b)^2$,
it asymptotically behaves like the objective
\begin{equation}\label{eq:info-nce-estimand-def}
    \AME(\g) = \mathbb{E}_{(\x_1,\x_2) \sim p_{\x_1,\x_2}} 
    \left[ \lVert \g(\x_1) - \g(\x_2) \rVert_2 \right] - H(\g(\x))  
\end{equation}
for a single encoder $\g$, when $K \to \infty$~\citep{Wang2020,Kuegelgen2021}.
 
In the setting with two heterogeneous modalities, it is natural to employ
separate encoders $\g_1 \not= \g_2$, which can represent different
architectures.  Further, it is common to use a symmetrized version of the
objective \citep[e.g., see][]{Zhang2020,Radford2021}, which can be obtained by
computing the mean of the loss in both directions:
\begin{equation}\label{eq:sym-info-nce-def}
  \SymInfoNCE(\g_1, \g_2) =
    \nicefrac{1}{2}\InfoNCE(\g_1,\g_2) + 
    \nicefrac{1}{2}\InfoNCE(\g_2,\g_1). 
\end{equation}
Akin to \Cref{eq:info-nce-estimand-def}, we can approximate the symmetrized
objective for $\tau = 1$ and $\simi(a, b) = -(a - b)^2$, with a large number of
negative samples ($K \to \infty$), as follows:
\begin{equation}\label{eq:sym-info-nce-estimand-def}
    \SymAME(\g_1,\g_2) = \mathbb{E}_{(\x_1,\x_2) \sim p_{\x_1,\x_2}} \left[ \lVert \g_1(\x_1) - \g_2(\x_2) \rVert_2 \right] - \nicefrac{1}{2} \left(H(\g_1(\x_1)) + H(\g_2(\x_2)) \right).
\end{equation}
Since the similarity measure is symmetric, the approximation of the alignment
term is identical for both $\InfoNCE(\g_1,\g_2)$ and $\InfoNCE(\g_2,\g_1)$.
Each entropy term is approximated via the denominator of the respective loss
term, which can be viewed as a nonparametric entropy estimator
\citep{Wang2020}.  For our experiments, we employ the finite-sample estimators
$\InfoNCE$ and $\SymInfoNCE$, while for our theoretical analysis we use the
estimand $\SymAME$ to derive identifiability results.

\section{Generative process}
\label{sec:generative_process}

\begin{figure}
\centering
\begin{tikzpicture}
  \node[obs] (X1) {$\x_1$};%
  \node[latent,above=of X1,xshift=-1cm] (M1) {$\m_1$};%
  \node[latent,above=of X1,xshift=0cm] (S1) {$\s$}; %
  \node[latent,above=of X1,xshift=1cm] (C) {$\c$}; %
  \node[obs, xshift=2cm] (X2) {$\x_2$};%
  \node[latent,above=of X2,xshift=0cm] (S2) {$\ts$}; %
  \node[latent,above=of X2,xshift=1cm] (M2) {$\m_2$};%
  \edge{M1,C,S1}{X1}%
  \edge{M2,C,S2}{X2}%
  \edge{C}{S1}%
  \edge[bend left=50]{S1}{S2}%
  
  \newlength{\grouppad}
  \setlength{\grouppad}{0.6cm}
  \newlength{\groupshift}
  \setlength{\groupshift}{0.1cm}
  \coordinate[below = \grouppad of M1.center, yshift=-\groupshift]  (a1);
  \coordinate[above = \grouppad of M1.center, yshift=-\groupshift]  (a2);
  \coordinate[above = \grouppad of C.center, yshift=-\groupshift] (a3);
  \draw[dashed, black] (a1) arc (-90:-270:\grouppad)
		-- (a3) arc (90:-90:\grouppad)
		-- (a1);
  \draw[dashed] (a2) -- ++(90:0.2) node[above, outer sep = 0pt, inner sep = 0pt, text = black] {$\z_1$};
  \coordinate[below = \grouppad of C.center, yshift=\groupshift]  (b1);
  \coordinate[above = \grouppad of M2.center, yshift=\groupshift]  (b2);
  \coordinate[above = \grouppad of M2.center, yshift=\groupshift] (b3);
  \draw[dashed, lightgray] (b1) arc (-90:-270:\grouppad)
		-- (b3) arc (90:-90:\grouppad)
		-- (b1);
\draw[lightgray, dashed] (b2) -- ++(90:0.2) node[above, outer sep = 0pt, inner sep = 0pt, text = lightgray] {$\z_2$};
\end{tikzpicture}
\caption{Illustration of the multimodal generative process. Latent variables are denoted by clear nodes and observations by shaded nodes. We partition the latent space into $\z_1 = (\c, \s, \m_1)$ and $\z_2 = (\tc, \ts, \m_2)$, where $\tc = \c$ almost everywhere (\Cref{as:content}) and hence we consider only $\c$. Further, $\ts$ is a perturbed version of $\s$ (\Cref{as:style}) and $\m_1$, $\m_2$ are modality-specific variables. The observations $\x_1$ and $\x_2$ are generated by two distinct mixing functions $\f_1 \not = \f_2$, which are applied to the subsets of latent variables $\z_1$ and $\z_2$ respectively.}
\label{fig:lvm}
\end{figure}
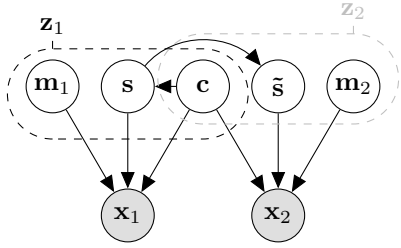

In the following, we formulate the multimodal generative process as a latent
variable model (\Cref{sec:lvm}) and then specify our technical
assumptions on the relation between modalities (\Cref{sec:relation}).

\subsection{Latent variable model}
\label{sec:lvm}

On a high level, we assume that there exists a continuous random variable $\z$
that takes values in the latent space $\mathcal{Z} \subseteq \mathbb{R}^n$,
which contains all information to generate observations of both modalities.%
\footnote{%
  Put differently, we assume that all observations lie on a continuous manifold, which can have
  much smaller dimensionality than the observation space of the respective modality.
} 
Moreover, we assume that $\z = (\c,\s,\m_1,\m_2)$ can be uniquely
partitioned into four disjoint parts:
\begin{enumerate}[\itshape(i)]
  \item an invariant part $\c$ which is always shared across modalities,
    and which we refer to as \emph{content};
  \item a variable part $\s$ which may change across modalities, and which we
    refer to as \emph{style};
  \item two modality-specific parts, $\m_1$ and $\m_2$, each of which is unique
    to the respective modality.
\end{enumerate}

Let $\z_1 = (\c, \s, \m_1)$ and $\z_2 = (\tc, \ts, \m_2)$, where $\tc = \c$
almost everywhere and $\ts$ is generated by perturbations that are specified in
\Cref{sec:relation}.  Akin to multi-view ICA (Equation~\ref{eq:multiview_ica}),
we define the generative process for modalities $\x_1$ and $\x_2$ as follows:
\begin{equation}
\label{eq:generative-process}
  \z \sim p_{\z}, \quad \x_1 = \f_1(\z_1), \quad \x_2 = \f_2(\z_2),
\end{equation}
where $\f_1: \mathcal{Z}_1 \to \mathcal{X}_1$ and $\f_2: \mathcal{Z}_2 \to
\mathcal{X}_2$ are two smooth and invertible mixing functions with smooth
inverse (i.e., diffeomorphisms) that generate observations $\x_1$ and $\x_2$
taking values in $\mathcal{X}_1 \subseteq \mathbb{R}^{d_1}$ and $\mathcal{X}_2
\subseteq \mathbb{R}^{d_2}$ respectively.  Generally, we assume that
observations from different modalities are generated by distinct mechanisms
$\f_1 \not = \f_2$ that take modality-specific latent variables as input. As
for the multi-view setting~\citep{Kuegelgen2021}, the considered generative
process  \emph{goes beyond the classical ICA setting} by allowing for
statistical dependencies within blocks of variables (e.g., between dimensions
of $\c$) and also for causal dependencies from content to style, as illustrated
in \Cref{fig:lvm}. We assume that $p_\z$ is a smooth density that factorizes as
$p_\z = p_\c \, p_{\s|\c} \, p_{\m_1}p_{\m_2}$ in the causal setting, and as
the product of all involved marginals when there is no causal dependence from
$\c$ to $\s$.

The outlined generative process is fairly general and it applies to a wide
variety of practical settings.  The content invariance describes a shared
phenomenon that is not directly observed but manifests in the observations from
both modalities. Style changes describe shared influences that are not robust
across modalities, e.g., non-invertible transformations such as data
augmentations, or non-deterministic effects of an unobserved confounder.
Modality-specific factors can be viewed as variables that describe the inherent
heterogeneity of each modality (e.g., background noise).

\subsection{Relation between modalities}
\label{sec:relation}

Next, we specify our assumptions on the relation between modalities by defining
the conditional distribution $p_{\z_2|\z_1}$, which describes the relation
between latent variables $\z_1$ and $\z_2$, from which observations $\x_1$ and
$\x_2$ are generated via \Cref{eq:generative-process}.  Similar to previous
work in the multi-view setting~\citep{Kuegelgen2021}, we assume that content is
invariant, i.e., $\tc = \c$ almost everywhere (\Cref{as:content}), and that
$\ts$ is a perturbed version of $\s$ (\Cref{as:style}). To state our
assumptions for the multimodal setting, we also need to consider the
modality-specific latent variables.

\begin{assumption}[Content-invariance]
\label{as:content}
The conditional density $p_{\z_2|\z_1}$ over $\mathcal{Z}_2 \times \mathcal{Z}_1$
takes the form
\begin{equation} \label{eq:z2-given-z1}
	p_{\z_2|\z_1}(\z_2|\z_1) = \delta(\tc - \c) p_{\ts|\s}(\ts|\s) p_{\m_2}(\m_2)
\end{equation}
for some continuous density $p_{\ts|\s}$ on $\mathcal{S} \times \mathcal{S}$,
where $\delta( \cdot )$ is the Dirac delta function, i.e., $\tc = \c$~a.e. 
\end{assumption}

To fully specify $p_{\z_2|\z_1}$, it remains to define the style changes, which
are described by the conditional distribution $p_{\ts|\s}$. There are several
justifications for modeling such a stochastic relation between $\s$ and
$\ts$~\citep{Zimmermann2021,Kuegelgen2021}; one could either consider $\ts$ to
be a noisy version of $\s$, or consider $\ts$ to be the result of an
augmentation that induces a soft intervention on $\s$.%
\footnote{%
  Note that the asymmetry between $\z_1$ and $\z_2$ (or between $\s$ and $\ts)$
  is not strictly required.  We chose to write $\z_2$ as a perturbation of
  $\z_1$ to simplify the notation and for consistency with previous work.
  Instead, we could model \emph{both} $\z_1$ and $\z_2$ via perturbations of
  $\z$, as described in \Cref{app:symmetric_generative_process}.
}

\begin{assumption}[Style changes]
\label{as:style}
  Let $\mathcal{A}$ be the powerset of style variables $\{1, \dots, n_s \}$ and
  let $p_A$ be a distribution on $\mathcal{A}$. Then, the style conditional
  $p_{\ts|\s}$ is obtained by conditioning on a set $A$:
  \begin{equation}
      p_{\ts|\s}(\ts|\s) = \sum_{A \in \mathcal{A}} p_A(A) 
      \left( \delta(\ts_{A^c}-\s_{A^c}) p_{\ts_A|\s_A}(\ts_A|\s_A) \right)
  \end{equation}
  where $p_{\ts_A|\s_A}$ is a continuous density on $\mathcal{S}_A \times
  \mathcal{S}_A$, $\mathcal{S}_A \subseteq \mathcal{S}$ denotes the subspace of
  changing style variables specified by $A$, and $A^c = \{ 1, \dots, n_s \}
  \backslash A$ denotes the complement of $A$.
  Further, for any style variable $l \in \{ 1, \dots, n_s \}$, there exists a set
  $A \subseteq \{ 1, \dots, n_s \}$ with $l \in A$, s.t.
  \begin{enumerate}[\itshape(i)]
      \item $p_{A}(A) > 0$,
      \item $p_{\ts_A|\s_A}$ is smooth w.r.t.~both $\s_A$ and $\ts_A$, and
      \item for any $\s_A$, $p_{\ts_A|\s_A}( \cdot | \s_A) > 0$, in some open
        non-empty subset containing $\s_A$.
  \end{enumerate}
\end{assumption}
Intuitively, to generate a pair of observations $(\x_1, \x_2)$, we
independently flip a biased coin for each style dimension to select a subset of
style features $A \subseteq \{ 1, \dots, n_s \}$, which are jointly perturbed
to obtain $\ts$. Condition \emph{(i)} ensures that every style dimension has a
positive probability to be perturbed,%
\!\footnote{%
  If a style variable would be perturbed with zero probability, it would be a
  content variable.
} 
while \emph{(ii)} and \emph{(iii)} are technical smoothness
conditions that will be used for the proof of \Cref{th:main}.

Summarizing, in this section we have formalized the multimodal generative
process as a latent variable model (\Cref{sec:lvm}) and specified our
assumptions on the relation between modalities via the conditional distribution
$p_{\z_1 | \z_2}$ (\Cref{sec:relation}). Next, we segue into the topic of
representation learning and show that, for the specified generative process,
multimodal contrastive learning can identify the content factors up to a
block-wise indeterminacy.

\section{Identifiability Results}
\label{sec:identifiability_results}

First, we need to define block-identifiability \citep{Kuegelgen2021} for the
multimodal setting in which we consider modality-specific mixing functions and
encoders. In the following, $n_c$ denotes the number of content variables and
the subscript $1{:}n_c$ denotes the subset of content dimensions (indexed from
1 to $n_c$ w.l.o.g.).

\begin{definition}[Block-identifiability]
\label{def:block-identified}
  The true content partition $\c = \f_1^{-1}(\x_1)_{1:n_c} = \f_2^{-1}(\x_2)_{1:n_c}$
  is block-identified by a function $\g_i: \mathcal{X}_i \to \mathcal{Z}_i$, 
  with $i \in \{ 1,2 \}$, if there exists an invertible function
  $\h_i: \mathbb{R}^{n_c} \to \mathbb{R}^{n_c}$, s.t.~for the inferred content
  partition $\hat{\c}_i = \g_i(\x_i)_{1:n_c}$ it holds that $\hat{\c}_i = \h_i(\c)$.
\end{definition}
It is important to note that block-identifiability does not require the
identification of \emph{individual} factors, which is the goal in multi-view
nonlinear ICA \citep{Gresele2019,Locatello2020,Zimmermann2021,Klindt2021} and
the basis for strict definitions of disentanglement
\citep{Bengio2013,Higgins2018,Shu2020}. Instead, our goal is to isolate the
group of invariant factors (i.e., the content partition) from the remaining
factors of variation in the data.

Specifically, our goal is to show that contrastive learning can block-identify
the content variables for the multimodal setting described in
\Cref{sec:generative_process}. We formalize this in \Cref{th:main} and thereby
relax the assumptions from previous work by allowing for distinct generating
mechanisms $\f_1 \neq \f_2$ and additional modality-specific latent variables.
\begin{restatable}{theorem}{thmain}
\label{th:main}
  Assume the data generating process described in Sec.~\ref{sec:lvm}, i.e.~data
  pairs $(\x_1, \x_2)$ generated from \Cref{eq:generative-process} with
  $p_{\z_1} = p_{\z \backslash \{ \m_2 \} }$ and $p_{\z_2|\z_1}$ as defined in
  Assumptions~\ref{as:content} and~\ref{as:style}. Further, assume that
  $p_{\z}$ is a smooth and continuous density on $\mathcal{Z}$ with $p_{\z}(\z) > 0$
  almost everywhere. Let $\g_1: \mathcal{X}_1 \to (0,1)^{n_c}$ and 
  $\g_2: \mathcal{X}_2 \to (0,1)^{n_c}$ be smooth functions that minimize $\SymAME$ as
  defined in Eq.~(\ref{eq:sym-info-nce-estimand-def}). Then, $\g_1$ and
  $\g_2$ block-identify the true content variables in the sense of
  Def.~\ref{def:block-identified}. 
\end{restatable}

A proof of \Cref{th:main} is provided in \Cref{app:proof_of_thmain}.
Intuitively, the result states that contrastive learning can identify the
content variables up to a block-wise indeterminacy. Similar to previous work,
the result is based on the optimization of the asymptotic form of the
contrastive loss (Equation~\ref{eq:sym-info-nce-estimand-def}). Moreover,
\Cref{th:main} assumes that the number of content variables is known or that it
can be estimated (e.g., with a heuristic like the elbow method). We address the
question of selecting the encoding size with dimensionality ablations
throughout our experiments. In \Cref{sec:discussion}, we will return to the
discussion of the assumptions in the context of the experimental results.

\section{Experiments}

The goal of our experiments is to test whether contrastive learning can
block-identify content in the multimodal setting, as described by
\Cref{th:main}. First, we verify identifiability in a fully controlled setting
with numerical simulations (\Cref{subsec:numerical_experiment}).  Second, we
corroborate our findings on a complex multimodal dataset of image/text pairs
(\Cref{subsec:imagetext_experiment}). The code is provided in our github
repository.%
\footnote{\url{https://github.com/imantdaunhawer/multimodal-contrastive-learning}.}

\subsection{Numerical Simulation}
\label{subsec:numerical_experiment}

\begin{table}[t]
  \centering
  \begin{subtable}{.45\textwidth}
    \centering
    \resizebox{!}{1.2cm}{%
      \small
      \begin{tabular}{ccccc}
        \toprule
        \multicolumn{3}{c}{\textbf{Generative process}} & \multicolumn{2}{c}{$\bm{R^2}$ \textbf{(nonlinear)}}  \\
        \cmidrule(r){1-3}\cmidrule(r){4-5}
        \textbf{p(chg.)} & \textbf{Stat.} & \textbf{Cau.} & \textbf{Content $\c$} & \textbf{Style $\s$} \\
        \midrule
       1.0 & \xmark & \xmark  & $\textbf{1.00} \pm 0.00$ & $0.00 \pm 0.00$ \\
       0.75 & \xmark & \xmark & $\textbf{0.99} \pm 0.01$ & $0.00 \pm 0.00$ \\
       0.75 & \cmark & \xmark & $\textbf{0.99} \pm 0.00$ & $0.52 \pm 0.09$ \\
       0.75 & \xmark & \cmark & $\textbf{1.00} \pm 0.00$ & $\textbf{0.79} \pm 0.04$ \\
       0.75 & \cmark & \cmark & $\textbf{0.99} \pm 0.01$ & $\textbf{0.81} \pm 0.04$ \\
        \bottomrule
      \end{tabular}
    }  %
    \caption{Original setting}
  \label{subtab:numerical_original}
  \end{subtable}
  \begin{subtable}{.54\textwidth}
    \centering
    \resizebox{!}{1.2cm}{%
      \small
      \begin{tabular}{cccccc}
        \toprule
        \multicolumn{3}{c}{\textbf{Generative process}} & \multicolumn{3}{c}{$\bm{R^2}$ \textbf{(nonlinear)}}  \\
        \cmidrule(r){1-3}\cmidrule(r){4-6}
        \textbf{p(chg.)} & \textbf{Stat.} & \textbf{Cau.} & \textbf{Content $\c$} & \textbf{Style $\s$} & \textbf{Modality $\m_i$}\\
        \midrule
       1.0 & \xmark & \xmark  & $\textbf{0.99} \pm 0.00$ & $0.00 \pm 0.00$ & $0.00 \pm 0.00$ \\
       0.75 & \xmark & \xmark & $\textbf{1.00} \pm 0.00$ & $0.00 \pm 0.00$ & $0.00 \pm 0.00$ \\
       0.75 & \cmark & \xmark & $\textbf{0.95} \pm 0.01$ & $0.56 \pm 0.23$ & $0.00 \pm 0.00$ \\
       0.75 & \xmark & \cmark & $\textbf{0.98} \pm 0.00$ & $\textbf{0.87} \pm 0.04$ & $0.00 \pm 0.00$ \\
       0.75 & \cmark & \cmark & $\textbf{0.95} \pm 0.03$ & $\textbf{0.89} \pm 0.07$ & $0.00 \pm 0.00$ \\
        \bottomrule
      \end{tabular}
    }  %
    \caption{Multimodal setting}
  \label{subtab:numerical_multimodal}
  \end{subtable}
  \caption{%
    Results of the numerical simulations. We compare the original setting ($\f_1
    = \f_2$, left table) with the multimodal setting ($\f_1 \not= \f_2$, right
    table). Only the multimodal setting includes modality-specific latent
    variables.  Each row presents the results of a different setup with varying
    style-change probability p(chg.) and possible statistical (Stat.) and/or
    causal (Caus.) dependencies. Each value denotes the $R^2$ coefficient of
    determination (averaged across 3 seeds) for a nonlinear regression model that
    predicts the respective ground truth factor ($\c, \s$, or $\m_i$) from the
    learned representation.
  }
\label{tab:numerical_m1m2}
\end{table}

We extend the numerical simulation from \citet{Kuegelgen2021} and implement the
multimodal setting using modality-specific mixing functions ($\f_1 \not= \f_2$)
with modality-specific latent variables. The numerical simulation allows us to
measure identifiability with full control over the generative process.  The
data generation is consistent with the generative process described in
\Cref{sec:generative_process}.  We sample $\c \sim \mathcal{N}(0,
\Sigma_{\c})$, $\m_i \sim \mathcal{N}(0, \Sigma_{\m_i})$, and $\s \sim
\mathcal{N}(\a + B\c, \Sigma_{\s})$.  Statistical dependencies within blocks
(e.g., among components of $\c$) are induced by non-zero off-diagonal entries
in the corresponding covariance matrix (e.g., in $\Sigma_{\c}$).  To induce a
causal dependence from content to style, we set $a_i, B_{ij} \sim
\mathcal{N}(0, 1)$; otherwise, we set $a_i, B_{ij} = 0$.  For style changes,
Gaussian noise is added with probability $\pi$ independently for each style
dimension: $\ts_i = \s_i + \epsilon$, where $\epsilon \sim \mathcal{N}(0,
\Sigma_{\epsilon})$ with probability $\pi$.  We generate the observations $\x_1
= \f_1(\c, \s, \m_1)$ and  $\x_2 = \f_2(\c, \ts, \m_2)$ using two
\emph{distinct} nonlinear mixing functions, i.e, for each $i \in \{1, 2\}$,
$\f_i: \mathbb{R}^d \to \mathbb{R}^d$ is a separate, invertible 3-layer~MLP
with LeakyReLU activations.  We train the encoders for 300,000 iterations using
the symmetrized InfoNCE objective (Equation~\ref{eq:sym-info-nce-def}) and the
hyperparameters listed in \Cref{sec:app-details-to-experimental-setting}.  We
evaluate block-identifiability by predicting the ground truth factors from the
learned representation using kernel ridge regression and report the $R^2$
coefficient of determination on holdout data.

\paragraph{Results}
We compare the original setting ($\f_1 = \f_2$,
\Cref{subtab:numerical_original}) with the multimodal setting ($\f_1 \not=
\f_2$, \Cref{subtab:numerical_multimodal}) and find that content can be
block-identified in \emph{both} settings, as the $R^2$ score is close to one
for the prediction of content, and quasi-random for the prediction of style and
modality-specific information.  Consistent with previous work, we observe that
some style information can be predicted when there are statistical and/or
causal dependencies; this is expected because statistical dependencies decrease
the effective dimensionality of content, while the causal dependence $\c \to
\s$ makes style partially predictable from the encoded content information.
Overall, the results of the numerical simulation are consistent with our
theoretical result from \Cref{th:main}, showing that contrastive learning can
block-identify content in the multimodal setting.

\subsection{Image/text pairs}
\label{subsec:imagetext_experiment}

Next, we test whether block-identifiability holds in a more realistic setting
with image/text pairs---two complex modalities with distinct generating
mechanisms. We extend the \emph{Causal3DIdent} dataset
\citep{Kuegelgen2021,Zimmermann2021}, which allows us to measure and control
the ground truth latent factors used to generate complex observations.  We use
\textit{Blender} \citep{Blender} to render high-dimensional images that depict
a scene with a colored object illuminated by a differently colored spotlight
and positioned in front of a colored background.  The scene is defined by 11
latent factors: the shape of the object (7 classes), position of the object
($x, y, z$ coordinates), orientation of the object ($\alpha, \beta, \gamma$
angles), position of the spotlight ($\theta$ angle), as well as the color of
the object, background, and spotlight respectively (one numerical value for
each).

\begin{wrapfigure}{r}{0.43\textwidth} 
\vspace{-18pt}
  \begin{center}
    \includegraphics[width=0.43\textwidth]{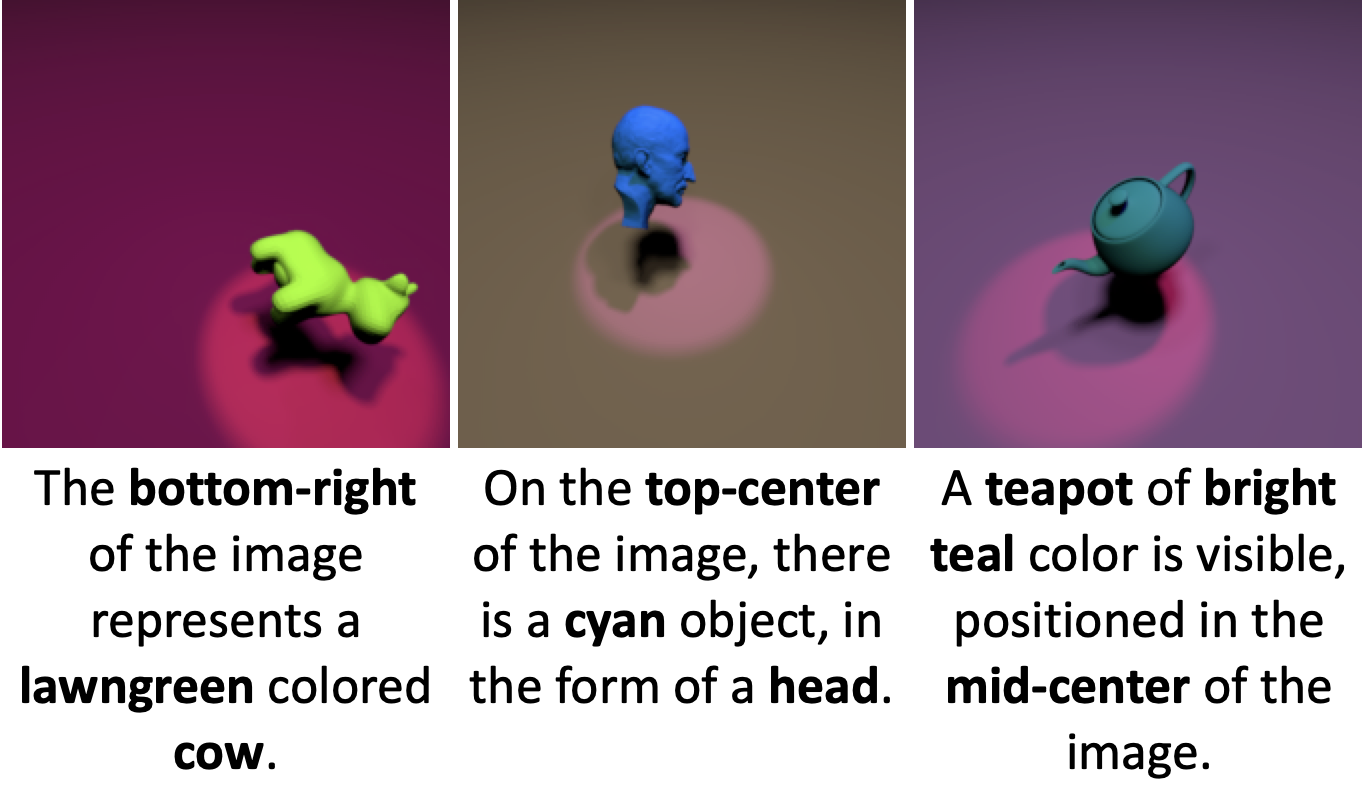}
    \vspace{-15pt}
    \caption{Examples of image/text pairs.}
    \label{fig:imagetext_examples}
  \end{center}
\vspace{-13pt}
\end{wrapfigure} 

\paragraph{\emph{Multimodal3DIdent}}
We extend the \emph{Causal3DIdent} dataset to the multimodal setting as
follows. We generate textual descriptions from the latent factors by adapting
the text rendering from the \emph{CLEVR} dataset \citep{Johnson2017}. Each
image/text pair shares information about the shape of the object (cow,
teapot,~etc.) and its position in the scene (e.g.,~bottom-right). For each
position factor, we use three clearly discernable values (top/center/bottom;
left/center/right), which can be described in text more naturally than
coordinates.  While shape and position are always shared (i.e., content)
between the paired image and text, the color of the object is causally
influenced by position and is stochastically shared (i.e., style). For the
object color, we use a continuous hue value, whereas for the text we match the
RGB value with the nearest value from a given palette (i.e., a list of named
colors, such as brown, beige, olive, etc.). The color palette is randomly
sampled from a set of three palettes to ensure the object color depicted in the
image does not uniquely determine the color described in the text. As
modality-specific factors for the images, we have object rotation, spotlight
position, and background color, while for the textual descriptions, we follow
\citet{Johnson2017} and use 5 different types of phrases to introduce
modality-specific variation. Examples of image/text pairs are shown in
\Cref{fig:imagetext_examples}. Further details about the dataset are provided
in \Cref{sec:app-details-to-experimental-setting}.

We train the encoders for 100,000 iterations using the symmetrized InfoNCE
objective (Equation~\ref{eq:sym-info-nce-def}) and the hyperparameters listed
in \Cref{sec:app-details-to-experimental-setting}. For the image encoder we use
a ResNet-18 architecture \citep{He2016} and for the text we use a convolutional
network. As for the numerical simulation, we evaluate block-identifiability by
predicting the ground truth factors from the learned representation. For
continuous factors, we  use kernel ridge regression and report the $R^2$ score,
whereas for discrete factors we report the classification accuracy of an MLP
with a single~hidden~layer. 

\paragraph{Results}
\Cref{fig:imagetext_results} presents the results on \emph{Multimodal3DIdent}
with a dimensionality ablation, where we vary the size of the encoding of the
model. Content factors (object position and shape) are always encoded well,
unless the encoding size is too small (i.e., smaller than 3-4 dimensions). When
there is sufficient capacity, style information (object color) is also encoded,
partly because there is a causal dependence from content to style and partly
because of the excess capacity, as already observed in previous work.
Image-specific information (object rotation, spotlight position, background
color) is mostly discarded, independent of the encoding size. Text-specific
information (phrasing) is encoded to a moderate degree (48--80\% accuracy),
which we attribute to the fact that phrasing is a discrete factor that violates
the assumption of continuous latents. This hints at possible limitations in the
presence of discrete latent factors, which we further investigate in
\Cref{app:additional_experiments} and discuss in \Cref{sec:discussion}.
Overall, our results suggest that contrastive learning can block-identify
content factors in a complex multimodal setting with image/text pairs.

\begin{figure*}[t]
    \centering
    \begin{subfigure}[t]{0.49\textwidth}
        \includegraphics[width=1.0\textwidth]{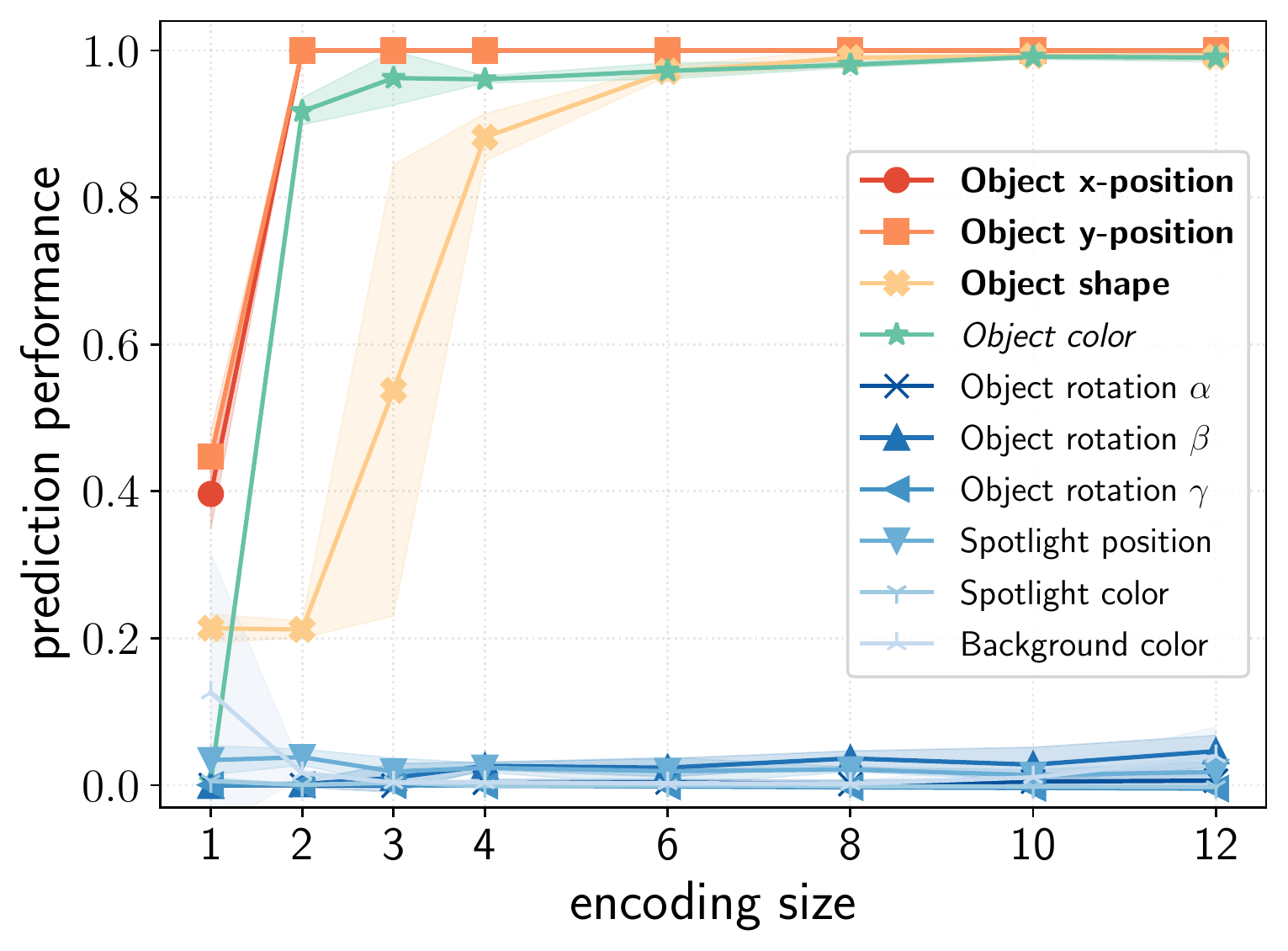}
        \caption{Prediction of image factors}
        \label{subfig:imagetext_results_image}
    \end{subfigure}%
    \hfill
    \begin{subfigure}[t]{0.49\textwidth}
        \centering
        \includegraphics[width=1.0\textwidth]{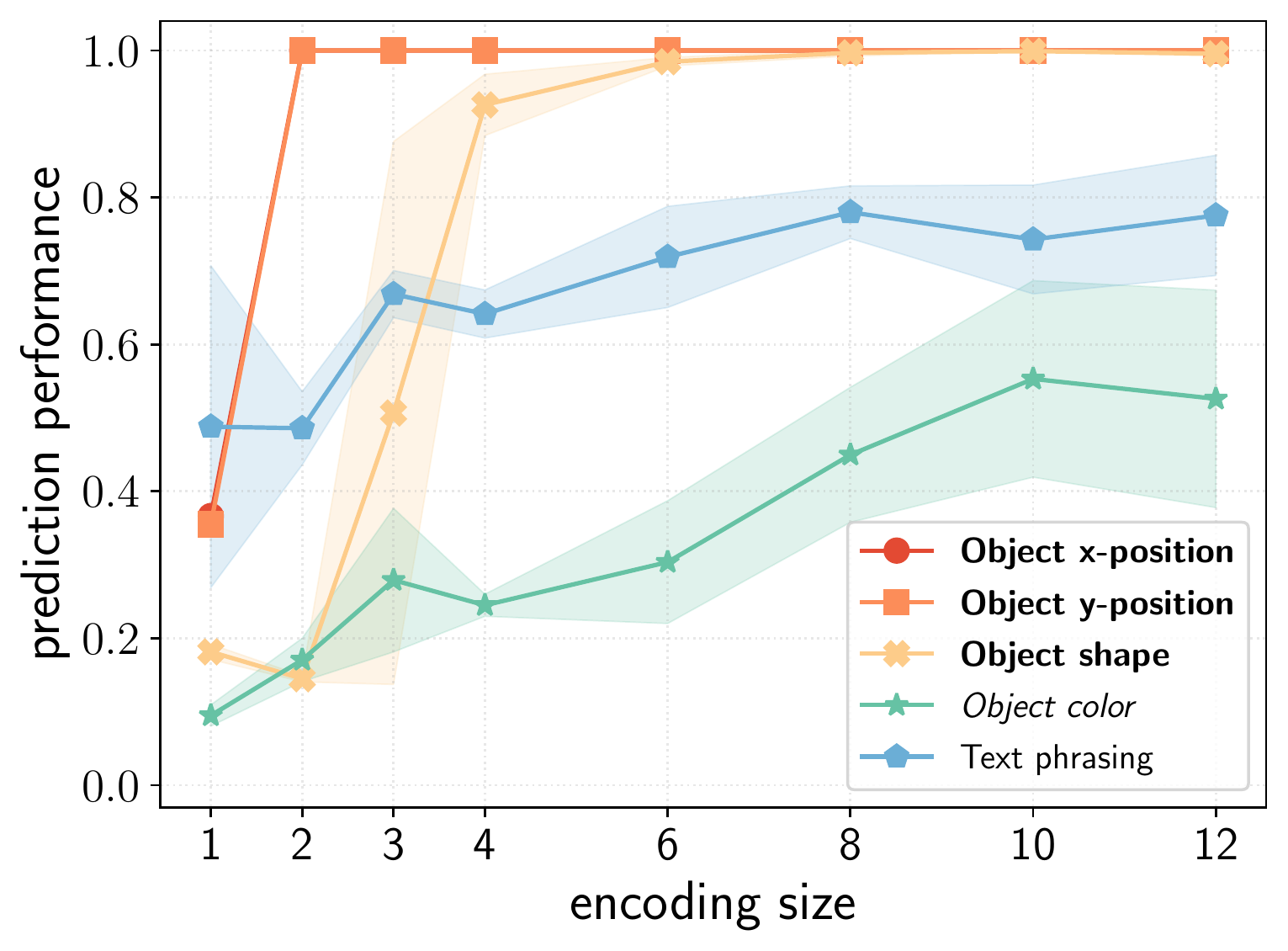}
        \caption{Prediction of text factors}
        \label{subfig:imagetext_results_text}
    \end{subfigure}
    \caption{Results on \emph{Multimodal3DIdent} as a function of the encoding
      size of the model. We assess the nonlinear prediction of ground truth
      image factors (left subplot) and text factors (right subplot) to quantify
      how well the learned representation encodes the respective factors.
      Content factors are denoted in bold and style factors in italic. Along
      the x-axis, we vary the encoding size, i.e., the output dimensionality of
      the model. We measure the prediction performance in terms of the $R^2$
      coefficient of determination for continuous factors and classification
      accuracy for discrete factors respectively. Each point denotes the
      average across three seeds and bands show one standard deviation. 
    }
\label{fig:imagetext_results}
\end{figure*}

\section{Related work}
\label{sec:related_work}

\paragraph{Multi-view nonlinear ICA}
The goal of multi-view nonlinear ICA is to identify latent factors shared
between different views, as described in
\Cref{subsec:preliminaries_identifiability}.  There is a thread of works
\citep{Gresele2019,Locatello2020} that recover the latent variable up to a
component-wise indeterminacy in a setting with mutually independent latent
components, or up to block-wise inderterminacies in the case of independent
groups of shared and view-specific components \citep{Lyu2020,Lyu2022}.  Beyond
the assumption of independent (groups of) components, there is a line of works
\citep{Kuegelgen2021,Kong2022} that partition the latent space into blocks of
invariant and blocks of changing components and show that the invariant
components can be identified up to a block-wise indeterminacy, even when there
are nontrivial dependencies between latent components.  Our work advances in
this direction and considers heterogeneous modalities with nontrivial
statistical and causal dependencies between latents. We prove that shared
factors can be block-identified in a novel setting with modality-specific
mixing functions and modality-specific latent variables.

\paragraph{Multimodal representation learning}
Multimodal representation learning seeks to integrate information from
heterogeneous sources into a joint representation
\citep{Baltrusaitis2019,Guo2019}.  There is a myriad of methods designed to
learn representations of multimodal data either directly or indirectly. Among
methods that learn representations indirectly, there are multimodal
autoencoders \citep{Ngiam2011,Geng2022,Bachmann2022,Aghajanyan2022} and a large
variety of multimodal generative models
\citep[e.g.,][]{Suzuki2016,Wu2018,Shi2019,Huang2018,Tsai2019_transformer,Ramesh2021}
that learn representations by backpropagation of different forms of
reconstruction and/or masked prediction error.  A more direct approach is taken
by decoder-free methods that maximize the similarity between the encodings of
different modalities. This class of methods includes nonlinear canonical
correlation analysis
\citep{Akaho2001,Bach2002,Andrew2013,Wang2016,Tang2017,Karami2021} as well as
multi-view and multimodal contrastive learning
\citep{Tian2019_cmc,Bachmann2019,Federici2020,Tsai2021,Radford2021,Poklukar2022}.
While all of the named methods aim to integrate information across modalities,
they do not answer the underlying question of identifiability, which our work
seeks to address.

\section{Discussion}
\label{sec:discussion}

\paragraph{Implications and scope}
We have shown that contrastive learning can block-identify shared factors in
the multimodal setting. Numerical simulations
(\Cref{subsec:numerical_experiment}) verify our main theoretical result
(\Cref{th:main}), showing that contrastive learning block-identifies content
information (\Cref{def:block-identified}), when the size of the encoding
matches the number of content factors.  Experiments on a complex dataset of
image/text pairs corroborate that contrastive learning can isolate content in a
more realistic setting and even under some violations of the assumptions
underlying \Cref{th:main}. In \Cref{app:additional_experiments}, we include
further experiments that test violations with discrete factors and
dimensionality ablations that examine the robustness and sample complexity. %
More generally, we observe that contrastive learning encodes invariant
information (i.e., content) very well across all settings.  When there is
sufficient capacity, stochastically shared information (i.e., style) is encoded
to a moderate degree, but without affecting the prediction of invariant
information.  For practice, our results suggest that contrastive learning
without capacity constraints can encode \emph{any} shared factor, irregardless
of whether the factor is truly invariant across modalities or if its effect on
the observations is confounded by noise or other factors.  This is in line with
the information-theoretic view \citep{Oord2018,Poole2019}, i.e., that
contrastive learning maximizes the mutual information between
representations---a measure of mutual dependence that quantifies \emph{any}
information that is shared. Our results demonstrate that the size of the
encoding can be reduced to learn a representation that recovers invariant
information, as captured by the notion of block-identifiability. In practice,
this can be leveraged for representation learning in settings of
content-preserving distribution shifts \citep{Mitrovic2021,Federici2021}, where
information relevant for a downstream task~remains~unchanged.

\paragraph{Limitations and outlook}
First, \Cref{th:main} suggests that only \emph{invariant} factors can be
block-identified. However, in practice, there can be pairs of observations for
which the invariance is inadvertently violated, e.g., due to measurement
errors, occlusions, or other mistakes in the data collection.  On the one hand,
such a violation can be viewed as a mere artifact of the data collection and
could be managed via interventions on the generative process, e.g., actions in
reinforcement learning  \citep{Lippe2022,Brehmer2022,Ahuja2022,Lachapelle2022}.
On the other hand, violations of the content-invariance blur the line between
content and style factors and it would be interesting to study identifiability
in a more general setting with \emph{only} stochastically shared factors.
Second, \Cref{th:main} assumes that the number of content factors is known or
that it can be estimated. In practice, this might not be a significant
limitation, since the number of content factors can be viewed as a single
hyperparameter \citep[e.g.,][]{Locatello2020}, though the design of suitable
heuristics is an interesting research direction. We explore the idea of
estimating the number of content factors in \Cref{app:additional_experiments}
\Cref{fig:model_selection}. Third, \Cref{th:main} assumes that all latent
factors are continuous. While this assumption prevails in related work
\citep{Hyvaerinen1999,Hyvarinen2016,Hyvaerinen2019,Gresele2019,Locatello2019,Locatello2020,Zimmermann2021,Kuegelgen2021,Klindt2021},
our results in \Cref{subfig:imagetext_results_text} indicate that in the
presence of discrete factors, some style or modality-specific information can
be encoded.  In \Cref{app:additional_experiments} \Cref{fig:discrete}, we
provide numerical simulations that support these findings.  Finally, our model
can be extended to more than two modalities---a setting for which there are
intriguing identifiability results \citep{Gresele2019,Schoelkopf2016} as well
as suitable learning objectives \citep{Tian2019_cmc,Lyu2022}. Summarizing, the
described limitations mirror the assumptions on the generative process
(\Cref{sec:generative_process}), which may be relaxed in future work.

\section{Conclusion}

We addressed the problem of identifiability for multimodal representation
learning and showed that contrastive learning can block-identify latent factors
shared between heterogeneous modalities. We formalize the multimodal generative
process as a novel latent variable model with modality-specific generative
mechanisms and nontrivial statistical and causal dependencies between latents.
We prove that contrastive learning can identify shared latent factors up to a
block-wise indeterminacy and therefore isolate invariances between modalities
from other changeable factors.  Our theoretical results are corroborated by
numerical simulations and on a complex multimodal dataset of image/text pairs.
More generally, we believe that our work will help in shaping a theoretical
foundation for multimodal representation learning and that further relaxations
of the presented generative process offer rich opportunities for future work.

\newpage
\section*{Acknowledgements}

ID was supported by the SNSF grant \textit{\#200021-188466}. EP was supported
by the grant \textit{\#2021-911} of the Strategic Focal Area ``Personalized
Health and Related Technologies (PHRT)'' of the ETH Domain (Swiss Federal
Institutes of Technology). Experiments were performed on the ETH~Zurich
Leonhard cluster. Special thanks to Kieran~Chin-Cheong for his support in the
early stages of the project as well as to Luigi~Gresele and Julius~von~Kügelgen
for helpful discussions.

\section*{Reproducibility statement}

For our theoretical statements, we provide detailed derivations and state the
necessary assumptions. The generative process is specified in
\Cref{sec:generative_process} and the assumptions for block-identifiability are
referenced in \Cref{th:main}. We test violations of the key assumptions with
suitable experiments (dimensionality ablations; discrete latent factors) and
discuss the limitations of our work in \Cref{sec:discussion}. Further, we
empirically verify our theoretical results with numerical simulations and on
complex multimodal data. To ensure empirical reproducibility, the results of
every experiment were averaged over multiple seeds and are reported with
standard deviations. Information about implementation details, hyperparameter
settings, and evaluation metrics are included in
\Cref{sec:app-details-to-experimental-setting}. Additionally, we publish the
code to reproduce the experiments.

\bibliographystyle{apalike}
\bibliography{references}

\newpage
\appendix

\section[Appendix A: Theory]{Theory}
\label{app:proofs}

\subsection{Proof of Theorem 1}
\label{app:proof_of_thmain}

\thmain*

\begin{proof}
To prove Theorem~\ref{th:main}, we follow the proof structure from
\citet[][Theorem~4.4]{Kuegelgen2021} and divide the proof into three steps.
First, we show that there exists a pair of smooth functions 
$\g_1^*, \g_2^*$ that attain the global minimum of $\SymAME$
(Eq.~\refeq{eq:sym-info-nce-estimand-def}). Further, in
Equations~(\refeq{eq:invariance-rel-step1}--\refeq{eq:hh2-unif}), we derive
invariance conditions that have to hold almost surely for any pair of smooth
functions $\g_1, \g_2$ attaining the global minimum of
Eq.~(\refeq{eq:sym-info-nce-estimand-def}).  In Step~2, we use the invariance
conditions derived in Step~1 to show by contradiction that \emph{any} pair of
smooth functions $\g_1, \g_2$ that attain the global minimum in
Eq.~(\refeq{eq:sym-info-nce-estimand-def}) can only depend on content and not
on style or modality-specific information.  In the third and final step, for
$\h_1 := \g_1 \circ \f_1$ and $\h_2 := \g_2 \circ \f_2$, we show that both
functions must be bijections and hence that $\c$ is block-identified by $\g_1$
and $\g_2$ respectively.

\paragraph{Step 1.} 
Recall the asymptotic form of the objective, as defined in
\Cref{eq:sym-info-nce-estimand-def}:
\footnotesize
\begin{equation}%
  \SymAME(\g_1,\g_2) = \mathbb{E}_{(\x_1,\x_2) \sim p_{\x_1,\x_2}} 
  \left[ \lVert \g_1(\x_1) - \g_2(\x_2) \rVert_2 \right] - \nicefrac{1}{2} 
  \left( H(\g_1(\x_1)) + H(\g_2(\x_2)) \right) \;. 
  \tag{\ref{eq:sym-info-nce-estimand-def}}
\end{equation}
\normalsize
The global minimum of $\SymAME$ is reached when the first term is minimized and
the second term is maximized. The first term is minimized when the encoders
$\g_1$ and $\g_2$ are perfectly aligned, i.e., when $\g_1(\x_1) = \g_2(\x_2)$
holds for all pairs $(\x_1, \x_2) \sim p_{\x_1, \x_2}$. The second term attains
its maximum when $\g_1$ and $\g_2$ map to a uniformly distributed random
variable on $(0,1)^{n_c}$ respectively.%
\footnote{%
  Note that we restrict the range of $\g_1$ and $\g_2$ to $(0,1)^{n_c}$ by
  definition merely to simplify the notation. Generally, the uniform
  distribution  $\mathcal{U}(a, b)$ is the maximum entropy distribution on the
  interval $[a, b]$.
}

To show that there \emph{exists} a pair of functions that minimize $\SymAME$,
let $\g_1^* := \bm{d}_1 \circ \f_{1,1:n_c}^{-1}$ and let $\g_2^* := \bm{d}_2
\circ \f_{2,1:n_c}^{-1}$, where the subscript $1{:}n_c$ indexes the subset of
content dimensions w.l.o.g.~and where $\bm{d}_1$ and $\bm{d}_2$ will be defined
using the Darmois construction~\citep{darmois:51:analyse,Hyvaerinen1999}.
First, recall that ${\f_1^{-1}(\x_1)_{1:n_c} = \c}$ and that
${\f_2^{-1}(\x_2)_{1:n_c} = \tc}$ by definition. Second, for 
$i \in \{ 1, 2 \}$, let us define $\bm{d}_i: \mathcal{C} \mapsto (0,1)^{n_c}$
using the Darmois construction, such that $\bm{d}_i$ maps $\c$ and $\tc$ to a
uniform random variable respectively. It follows that $\g_1^*, \g_2^*$ are
smooth functions, because any function $\bm{d}_i$ obtained via the Darmois
construction is smooth and $\f_1^{-1}, \f_2^{-1}$ are smooth as well (each
being the inverse of a smooth function).

Next, we show that the pair of functions $\g_1^*, \g_2^*$, as defined above,
attains the global minimum of the objective $\SymAME$. We have that
\footnotesize
\begin{align}
	\SymAME(\g_1^*,\g_2^*) &= \mathbb{E}_{(\x_1,\x_2) \sim p_{\x_1,\x_2}} 
    \left[ \lVert \g_1^*(\x_1) - \g_2^*(\x_2) \rVert_2 \right] - \nicefrac{1}{2} 
    \left( H(\g_1^*(\x_1)) + H(\g_2^*(\x_2)) \right) \\
	&= \mathbb{E}_{(\x_1,\x_2) \sim p_{\x_1,\x_2}} 
    \left[ \lVert \bm{d}_1(\c) - \bm{d}_2(\tc) \rVert_2 \right] - \nicefrac{1}{2} 
    \left( H(\bm{d}_1(\c)) + H(\bm{d}_2(\tc)) \right) \\
	&= 0 
\end{align}
\normalsize
where by Assumption~\ref{as:content}, $\c = \tc$ almost surely, which implies
that the first term is zero almost surely. Further, $\bm{d}_i$ maps $\c, \tc$
to uniformly distributed random variables on $(0,1)^{n_c}$, which implies that
the differential entropy of $\bm{d}_1(\c)$ and $\bm{d}_2(\tc)$ is zero, as
well. Consequently, there exists a pair of functions $\g_1^*, \g_2^*$ that
minimizes $\SymAME$.

Next, let $\g_1: \mathcal{X}_1 \mapsto (0,1)^{n_c}$ and  $\g_2: \mathcal{X}_2
\mapsto (0,1)^{n_c}$ be \emph{any} pair of smooth functions that attains the
global minimum of Eq.~(\refeq{eq:sym-info-nce-estimand-def}), i.e.,
\footnotesize
\begin{equation}
  \label{eq:any-minimizer-zero}
  \SymAME(\g_1, \g_2) = \mathbb{E}_{(\x_1,\x_2) \sim p_{\x_1,\x_2}} 
  \left[ \lVert \g_1(\x_1) - \g_2(\x_2) \rVert_2 \right] - \nicefrac{1}{2} 
  \left( H(\g_1(\x_1)) + H(\g_2(\x_2)) \right) = 0 \;. 
\end{equation}
\normalsize
Let $\h_1 := \g_1 \circ \f_1$ and $\h_2 := \g_2 \circ \f_2$, and notice that
both are smooth functions since all involved functions are smooth by
definition. Since \Cref{eq:any-minimizer-zero} is a global minimum, it implies
the following invariance conditions for the individual terms:
\footnotesize
\begin{align}
	\mathbb{E}_{(\x_1,\x_2) \sim p_{\x_1,\x_2}} \left[ \lVert \h_1(\z_1) - \h_2(\z_2) \rVert_2 \right] 
    &= 0 \label{eq:invariance-rel-step1} \\
	H(\h_1(\z_2)) &= 0 \label{eq:hh1-unif} \\
	H(\h_2(\z_2)) &= 0 \label{eq:hh2-unif}
\end{align}
\normalsize
Hence, $\h_1(\z_1) = \h_2(\z_2)$ must hold almost surely 
w.r.t.~$p_{\x_1, \x_2}$. Additionally, \Cref{eq:hh1-unif} (resp.
\Cref{eq:hh2-unif}) implies that $\hat{\c}_1 = \h_1(\z_1)$ 
(resp. $\hat{\c}_2 = \h_2(\z_2)$) must be uniform on $(0,1)^{n_c}$.

\paragraph{Step 2.} 
Next, we show that any pair of functions that minimize $\SymAME$ depend only on
content information. Since style is independent of $\m_1$ and $\m_2$, we first
show that $\h_1(\z_1)$ does not depend on $\m_1$, and that $\h_2(\z_2)$ does
not depend on $\m_2$. We then show that $\h_1$ and $\h_2$ also cannot depend on
style, based on a result from previous work.

First note, that we can exclude all degenerate solutions where $\g_1$ maps a
component of $\m_1$ to a constant, since $\g_1$ would not be invertible anymore
and such a solution would violate the invariance in Eq.~(\refeq{eq:hh1-unif}).
To prove a contradiction, suppose that, w.l.o.g., 
$\h_1(\c,\s, \m_1)_{1:n_{c}} := \h_1(\z_1)_{1:n_{c}}$ depends on some component
in $\m_1$ in the sense that the partial derivative of $\h_1(\z_1)_{1:n_{c}}$
w.r.t.~some modality-specific variable $m_{1,l}$ is non-zero for some point
$(\c^*, \s^*, \m_1^*) \in \mathcal{Z}_1$. Specifically, it implies that the
partial derivative $\nicefrac{\partial \h_1(\z_1)_{1:n_{c}}}{\partial m_{1,l}}$
is positive in a neighborhood around $(\c^*, \s^*, \m_1^*)$, which is a
non-empty open set, since $\h_1$ is smooth. On the other hand, due to the
independence of $\z_2$ and $\m_1$, the fact that $\h_2(\z_2)_{1:n_c}$ cannot
not depend on $\m_1$, and that $p(\z) > 0$ almost everywhere, we come to a
contradiction.  That is, there exists an open set of points with positive
measure, namely the neighbourhood around $(\c^*, \s^*, \m_1^*)$, on which
\footnotesize
\begin{equation}
  | (\h_1(\z_1)_{1:n_c} - \h_2(\z_2)_{1:n_c}) | > 0
\end{equation}
\normalsize
almost surely, which contradicts the invariance in \Cref{eq:invariance-rel-step1}.
The statement does not change, if we add further dependencies of $\h_1$ on
components of $\m_1$, or for $\h_2$ on components of $\m_2$, because $\m_1$ and
$\z_2$ are independent, and $\m_2$ and $\z_1$ are independent as well. Hence,
we show that \emph{any} encoder that minimizes the objective in
\Cref{eq:sym-info-nce-estimand-def} cannot depend on modality-specific
information.

Having established that neither $\h_1(\z_1)_{1:n_c}$, nor $\h_2(\z_2)_{1:n_c}$
can depend on modality-specific information, it remains to show that style
information is not encoded, as well. Leveraging \Cref{as:style}, we can show
that the strict inequality in \Cref{eq:invariance-rel-step1} has a positve
density if $\h_1(\z_1)_{1:n_c}$ or $\h_2(\z_2)_{1:n_c}$ was dependent on a
dimension in $\s$ respectively $\ts$, which would again lead to a violation of
the invariance derived in \Cref{eq:invariance-rel-step1}, as shown
in~\citet[][Proof of Theorem~4.2]{Kuegelgen2021}.

\paragraph{Step 3.} 
It remains to show that $\h_1, \h_2$ are bijections. We know that $\mathcal{C}$
and $(0,1)^{n_c}$ are simply connected and oriented $C^1$ manifolds, and we
have established in Step~1 that $\h_1$ and $\h_2$ are smooth and hence
differentiable functions. Since $p_\c$ is a regular density, the uniform
distributions w.r.t.~the pushthrough functions $\h_1$ and $\h_2$ are regular
densities. Thus, $\h_1$ and $\h_2$ are bijections~\citep[Proposition~5]{Zimmermann2021}

Step 3 concludes the proof. We have shown that for any pair of smooth functions
$\g_1, \g_2$ that attain the global minimum of
Eq.~(\refeq{eq:sym-info-nce-estimand-def}), we have that $\c$ is
\emph{block-identified} (Def.~\refeq{def:block-identified}) by $\g_1$ and
$\g_2$.
\end{proof}

\subsection{Symmetric generative process}
\label{app:symmetric_generative_process}

Throughout the main body of the paper, we described an asymmetric generating
mechanism, where $\z_2$ is a perturbed version of $\z_1$. Here, we will briefly
sketch out how our model and results can be adapted to a symmetric setting,
where \emph{both} $\z_1$ and $\z_2$ are generated as perturbations of $\z$.

Concretely, we would need to make small adjustments to
Assumptions~\ref{as:content} and \ref{as:style} as follows. We start with the
content invariance in \Cref{as:content}, which specifies how 
$\z_1 = (\tc_1, \ts_1, \tm_1)$ and $\z_2 = (\tc_2, \ts_2, \tm_2)$ are generated.

Let $i \in \{ 1,2 \} $. The conditional density $p_{\z_i|\z}$ over
$\mathcal{Z}_i \times \mathcal{Z}$ takes the form
\begin{equation} \label{eq:zi-given-z}
  p_{\z_i|\z}(\z_i|\z) = \delta(\tc_i - \c) \delta(\tm_i - \m_i) p_{\ts_i|\s}(\ts_i|\s) \; ,
\end{equation}
where $\delta( \cdot )$ is the Dirac delta function, i.e., $\tc_i = \c$ almost
everywhere, as well as $\tm_i = \m_i$ almost everywhere. Note that since 
$\tc_1 = \c$ a.e.~and $\c = \tc_2$ a.e.,~it follows that $\tc_1 = \tc_2$ almost
everywhere, which is a property that is needed in Step~1 in the proof of
\Cref{th:main}. In addition, it still holds that $\tm_i \Indep \z_j$, for $i, j
\in \{ 1,2 \}$ and $i \neq j$, which is needed in Step 2 of the proof to show
that modality-specific information is not encoded.

Lastly, we need to revisit \Cref{as:style}, for which both $\ts_1$ and $\ts_2$
would be generated through perturbations of $\s$ via the conditional
distribution $p_{\ts_i|\s}$ on $\mathcal{S} \times \mathcal{S}$, as described
in \Cref{as:style}, for each $i$ individually. As a small technical nuance, we
would need to specify the conditional generation of the perturbed style
variables $\ts_1$ and $\ts_2$ such that they are not perturbed in an identical
manner w.r.t.~$\s$. This can be ensured by, e.g., constraining $p_A$
appropriately to exclude the degenerate case where dimensions in $\ts_1$ and
$\ts_2$ are perfectly aligned---a case that needs to be excluded for Step~2 of
the proof of \Cref{th:main}.

\section[Appendix B: Experiments]{Experiments}
\label{app:experiment_details}

\subsection{Experimental details}
\label{sec:app-details-to-experimental-setting}

\paragraph{Numerical simulation}
The generative process is described in \Cref{subsec:numerical_experiment}.
Here, we provide additional information about the experiment. The invertible
MLP is constructed similar to previous work
\citep{Hyvarinen2016,Hyvaerinen2017,Zimmermann2021,Kuegelgen2021} by resampling
square weight matrices until their condition number surpasses a threshold
value. For the original setting~($\f_1 = \f_2$), we use one encoder~($\g_1 = \g_2$),
whereas for the multimodal setting~(${\f_1 \not= \f_2}$), we use distinct
encoders~($\g_1 \not= \g_2$) to mirror the assumption of distinct mixing
functions and because, in practice, the dimensionality of the observations can
differ across modalities.  In \Cref{tab:numerical_details}, we specify the main
hyperparameters for the numerical simulation.

\paragraph{Multimodal3DIdent}
Our dataset of image/text pairs is based on the code used to generate the
\emph{Causal3DIdent} \citep{Kuegelgen2021,Zimmermann2021} and \emph{CLEVR}
\citep{Johnson2017} datasets.  Images are generated using the \emph{Blender}
renderer \citep{Blender}. The rendering serves as a complex mixing function
that generates the images from 11 different parameters (i.e., latent factors)
that are listed in \Cref{tab:imgtxt_latent_factors}. To generate textual
descriptions, we adapt the text rendering from \emph{CLEVR} \citep{Johnson2017}
and use 5 different phrases to induce modality-specific variation. The latent
factors used to generate the text are also listed in
\Cref{tab:imgtxt_latent_factors}. The dependence between the image and text
modality is determined by three content factors (object shape, x-position, and
y-position) and one style factor (object color). For the object color in the
image, we use a continuous hue value, whereas for the text we match the RGB
value with the nearest color value from one of three different
palettes$^\text{\ref{footnote:color_palettes}}$ that is sampled uniformly at
random for each observation. Further, we ensure that there are no overlapping
color values across palettes by using a prefix for the respective palette
(e.g., ``xkcd:black'') when necessary. In \Cref{subsec:imagetext_experiment},
we use a version of the \emph{Multimodal3DIdent} dataset with a causal
dependence from content to style. Specifically, the color of the object depends
on its x-position. In particular, we split the range of hue values $[0, 1]$
into three equally sized intervals and associate each of these intervals with a
fixed x-position of the object. For instance, if x-position is ``left'', we
sample the hue value from the interval $[0, \nicefrac{1}{3}]$. Consequently,
the color of the object can be predicted to some degree from the position of
the object.  Samples of image/text pairs from the \emph{Multimodal3DIdent}
dataset are shown in \Cref{fig:imagetext_examples,fig:multimodal3dident_app}.
The hyperparameters for the experiment are listed in \Cref{tab:imgtxt_details}.
In \Cref{app:additional_experiments}, we provide additional results for a
version of the dataset with mutually independent factors. 

\paragraph{High-dimensional image pairs} 
In \Cref{app:additional_experiments}, we provide additional results using a
dataset of high-dimensional pairs of images of size $224\times224\times3$. Similar to
\textit{Multimodal3DIdent}, images are generated using \textit{Blender}
\citep{Blender} and code adapted from previous work
\citep{Zimmermann2021,Kuegelgen2021}. Each image depicts a scene with one type
of object \citep[a teapot, like in][]{Zimmermann2021} in front of a colored
background and illuminated by a colored spotlight (for examples, see
\Cref{fig:m3DIdent0}). The scene is defined by 9 continuous latent variables
each of which is sampled from a uniform distribution. Object positions (x-, y-
and z-coordinates) are content factors that are always shared between
modalities, while object-, spotlight- and background-colors are style factors
that are stochastically shared. Modality-specific factors are object rotation
($\alpha$ and $\beta$ angles) for one modality and spotlight position for the
other. To simulate modality-specific mixing functions, we render the objects
using distinct textures (i.e., rubber and metallic) for each modality. Further,
we generate two versions of this dataset, with and without causal dependencies.
For the dataset with causal dependencies we sample the latent factors according
to a causal model, where background-color depends on z-position and
spotlight-color depends on object-color.  We use ResNet-18 encoders and similar
hyperparameter values to those used for the image/text experiment
(\Cref{tab:imgtxt_details}). 

\begin{table}[!ht]
\vspace{2em}
\centering
    \begin{subtable}{.5\linewidth}
    \centering
    \small{%
        \begin{tabular}{lr}
          \toprule
          \textbf{Parameter} & \textbf{Value} \\
          \midrule
          Generating function & 3-layer MLP\\
          Encoder & 7-layer MLP\\
          Optimizer & Adam \\  %
          Cond. threshold ratio & 1e-3 \\
          Dimensionality $d$ & 15\\
          Batch size & 6144 \\
          Learning rate & 1e-4 \\
          Temperature $\tau$ & 1.0 \\
          \# Seeds & 3 \\
          \# Iterations & 300,000 \\
          Similarity metric & Euclidian \\
          Gradient clipping & 2-norm; max value 2 \\
          \bottomrule
        \end{tabular}
        \caption{Parameters used for the numerical simulation.}
        \label{tab:numerical_details}
    } %
    \end{subtable}%
    \begin{subtable}{.5\linewidth}
    \small{%
        \begin{tabular}{lr}
          \toprule
          \textbf{Parameter} & \textbf{Value} \\
          \midrule
          Generating function & Image and text rendering\\
          Image encoder & ResNet-18\\
          Text encoder & 4-layer ConvNet\\
          Optimizer & Adam \\  %
          Batch size & 256 \\
          Learning rate & 1e-5 \\
          Temperature $\tau$ & 1.0 \\
          \# Seeds & 3 \\
          \# Iterations & 100,000 \\
          \# Samples (train\,/\,val\,/\,test) & 125,000\,/\,10,000\,/\,10,000 \\
          Similarity metric & Cosine similarity \\
          Gradient clipping & 2-norm; max value 2 \\
          \bottomrule
        \end{tabular}
        \caption{Parameters used for \emph{Multimodal3DIdent}.}
        \label{tab:imgtxt_details}
    } %
    \end{subtable}
    \caption{Experimental parameters and hyperparameters used for the two experiments in the main text.}
\end{table}

\begin{table}[!ht]
\vspace{2em}
\small{%
\centering
    \begin{tabular}{lrr}
      \toprule
      \textbf{Latent factor} & \textbf{Distribution} & \textbf{Details} \\
      \midrule
      \textbf{Object shape} & Categorical & 7 unique values \\
      \textbf{Object x-position} & Categorical & 3 unique values \\
      \textbf{Object y-position} & Categorical & 3 unique values \\
      \emph{Object color} & Uniform & hue value in $[0, 1]$ \\
      Object rotation $\alpha$ & Uniform & angle value in $[0, 1]$ \\
      Object rotation $\beta$ & Uniform & angle value in $[0, 1]$ \\
      Object rotation $\gamma$ & Uniform & angle value in $[0, 1]$ \\
      Spotlight position & Uniform & angle value in $[0, 1]$ \\
      Spotlight color & Uniform & hue value in $[0, 1]$ \\
      Background color & Uniform & hue value in $[0, 1]$ \\
      \midrule
      \textbf{Object shape} & Categorical & 7 unique values \\
      \textbf{Object x-position} & Categorical & 3 unique values \\
      \textbf{Object y-position} & Categorical & 3 unique values \\
      \emph{Object color} & Categorical & color names (3 palettes)\footnotemark \\  %
      Text phrasing & Categorical & 5 unique values \\  %
      \bottomrule
    \end{tabular}
    \caption{Description of the latent factors used to generate
      \emph{Multimodal3DIdent}. The first 10 factors are used to generate the
      images and the remaining 5 factors are used to generate the text.
      Object~z-position is kept constant for all images, which is why we do not list
      it among the generative factors. Independent factors are drawn uniformly from
      the respective distribution. Content factors are denoted in bold and style
      factors in italic; the remaining factors are modality-specific.
    }
\label{tab:imgtxt_latent_factors}
} %
\end{table}

\footnotetext{%
\label{footnote:color_palettes}
  We use the following three palettes from the \texttt{matplotlib.colors} API:
  Tableau colors (10 values), CSS4 colors (148 values), and XKCD colors (949
  values).
}

\newpage
\clearpage
\subsection{Additional experimental results}
\label{app:additional_experiments}

\begin{figure}[t]
    \vspace{-2.0em}
    \centering
    \begin{subfigure}[t]{1.0\textwidth}
        \includegraphics[width=1.0\textwidth]{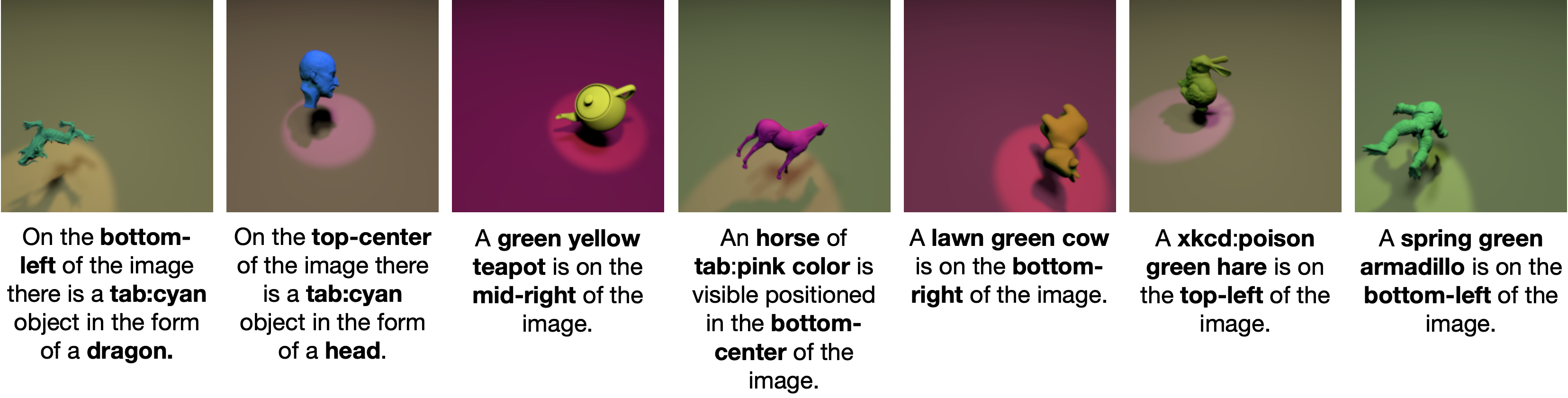}
        \vspace{-10pt}
    \end{subfigure}
    \vspace{-1.5em}
    \caption{Examples of image/text pairs from the \textit{Multimodal3DIdent}
      dataset. Each sample shows one of the seven shapes or classes of objects
      included in the dataset.
    }
\label{fig:multimodal3dident_app}
\end{figure}

\begin{table}[t]
\vspace{1.5em}
  \centering
  \begin{subtable}{.49\textwidth}
    \centering
    \resizebox{0.95\textwidth}{!}{%
      \small
      \begin{tabular}{ccccc}
        \toprule
        \multicolumn{3}{c}{\textbf{Generative process}} & \multicolumn{2}{c}{$\bm{R^2}$ \textbf{(nonlinear)}}  \\
        \cmidrule(r){1-3}\cmidrule(r){4-5}
        \textbf{p(chg.)} & \textbf{Stat.} & \textbf{Cau.} & \textbf{Content $\c$} & \textbf{Style $\s$} \\
        \midrule
       1.0 & \xmark & \xmark  & $\textbf{1.00} \pm 0.00$ & $0.00 \pm 0.00$ \\
       0.75 & \xmark & \xmark & $\textbf{0.99} \pm 0.01$ & $0.00 \pm 0.00$ \\
       0.75 & \cmark & \xmark & $\textbf{0.99} \pm 0.00$ & $0.52 \pm 0.09$ \\
       0.75 & \xmark & \cmark & $\textbf{1.00} \pm 0.00$ & $\textbf{0.79} \pm 0.04$ \\
       0.75 & \cmark & \cmark & $\textbf{0.99} \pm 0.01$ & $\textbf{0.81} \pm 0.04$ \\
        \bottomrule
      \end{tabular}
    }  %
    \caption{Original setting} %
    \label{subtab:numerical_ablation_nom1m2_vanilla}
  \end{subtable}
  \begin{subtable}{.49\textwidth}
    \centering
    \resizebox{0.95\textwidth}{!}{%
      \small
      \begin{tabular}{ccccc}
        \toprule
        \multicolumn{3}{c}{\textbf{Generative process}} & \multicolumn{2}{c}{$\bm{R^2}$ \textbf{(nonlinear)}}  \\
        \cmidrule(r){1-3}\cmidrule(r){4-5}
        \textbf{p(chg.)} & \textbf{Stat.} & \textbf{Cau.} & \textbf{Content $\c$} & \textbf{Style $\s$} \\
        \midrule
       1.0 & \xmark & \xmark  & $\textbf{1.00} \pm 0.00$ & $0.00 \pm 0.00$ \\
       0.75 & \xmark & \xmark & $\textbf{1.00} \pm 0.00$ & $0.00 \pm 0.00$ \\
       0.75 & \cmark & \xmark & $\textbf{0.99} \pm 0.01$ & $0.36 \pm 0.10$ \\
       0.75 & \xmark & \cmark & $\textbf{1.00} \pm 0.00$ & $\textbf{0.81} \pm 0.03$ \\
       0.75 & \cmark & \cmark & $\textbf{0.99} \pm 0.01$ & $\textbf{0.83} \pm 0.05$ \\
        \bottomrule
      \end{tabular}
    }  %
    \caption{Multimodal setting} %
  \label{subtab:numerical_ablation_nom1m2_multimodal}
  \end{subtable}
  \caption{%
    Results of the numerical simulations \emph{without} modality-specific latent
    variables. We compare the original setting ($\f_1 = \f_2$, left table) with
    the multimodal setting ($\f_1 \not= \f_2$, right table). Each row presents
    the results of a different setup with varying style-change probability
    p(chg.) and possible statistical (Stat.) and/or causal (Caus.) dependencies.
    Each value denotes the $R^2$ coefficient of determination (averaged across 3
    seeds) for a nonlinear regression model that predicts the respective ground
    truth factor ($\c, \s$, or $\m_i$) from the learned representation.
  }
\label{tab:numerical_ablation_nom1m2}
\end{table}

\paragraph{Numerical simulation without modality-specific latents}
Recall that the considered generative process (\Cref{sec:generative_process})
has two sources of modality-specific variation: modality-specific mixing
functions and modality-specific latent variables. To decouple the effect of
these two sources of variation, we conduct an ablation study \emph{without}
modality-specific latent variables. \Cref{tab:numerical_ablation_nom1m2}
presents the results, showing that content is block-identified in both the
original setting ($\f_1 = \f_2$,
\Cref{subtab:numerical_ablation_nom1m2_vanilla}) and the multimodal setting
($\f_1 \not= \f_2$, \Cref{subtab:numerical_ablation_nom1m2_multimodal}).
Compared to \Cref{tab:numerical_m1m2}, we observe that the content prediction
is improved slightly in the case without modality-specific latent variables.
Hence, our results suggest that contrastive learning can block-identify content
factors in the multimodal setting with and without modality-specific latent
variables.

\paragraph{Numerical simulation with discrete latent factors}
Extending the numerical simulation from \Cref{subsec:numerical_experiment}, we
test block-identifiability of content information when observations are
generated from a mixture of continuous and discrete latent variables, thus
violating one of the assumptions from \Cref{th:main}. In this setting, content,
style and modality-specific information are random variables with 5 components
sampled from either a continuous normal distribution or a discrete multinomial
distribution with $k$ classes, for which we experiment with different ${k \in
\{3,4,\ldots,10\}}$. For all settings, we train an encoder with the InfoNCE
objective and set the encoding size to 5 dimensions. The other hyperparameters
used in this set of experiments are detailed in \Cref{tab:numerical_details}.
To ensure convergence of the models, we extended the number of training
iterations to 600,000 and 3,000,000 for experiments with discrete
style/modality-specific and discrete content variables respectively. With
discrete style or modality-specific variables and continuous content
(\Cref{subfig:stylediscrete,subfig:msdiscrete}), the results suggest that
content is block-identified, since the prediction of style and
modality-specific information is at chance level (i.e., $\textit{accuracy} =
1/k$) while content is consistently fully recovered ($R^{2} \geq 0.99$). In the
opposite setting, with continuous style and modality-specific variables and
discrete content (\Cref{subfig:contentdiscrete}), the number of content classes
appears to be a critical factor for block-identifiability of content: while
content is always encoded well, style information is also encoded to a
significant extent when the number of content classes is small, but
significantly less style can be recovered when the number of content classes
increases. Through this set of experiments, we challenge the assumption that
\textit{all} generative factors should be continuous (c.f.,
\Cref{sec:generative_process}) and show that block-identifiability of content
can still be satisfied when content is continuous while style or
modality-specific variables are discrete. On the other hand, style is encoded
to a significant extent when content is discrete, which might explain our
observation for the image/text experiment, where we saw that, in the presence
of discrete content factors, some style information can be encoded.

\begin{figure}[t]
\vspace{-0em}
    \centering
    \begin{subfigure}[t]{0.3\textwidth}
        \includegraphics[width=1.0\textwidth]{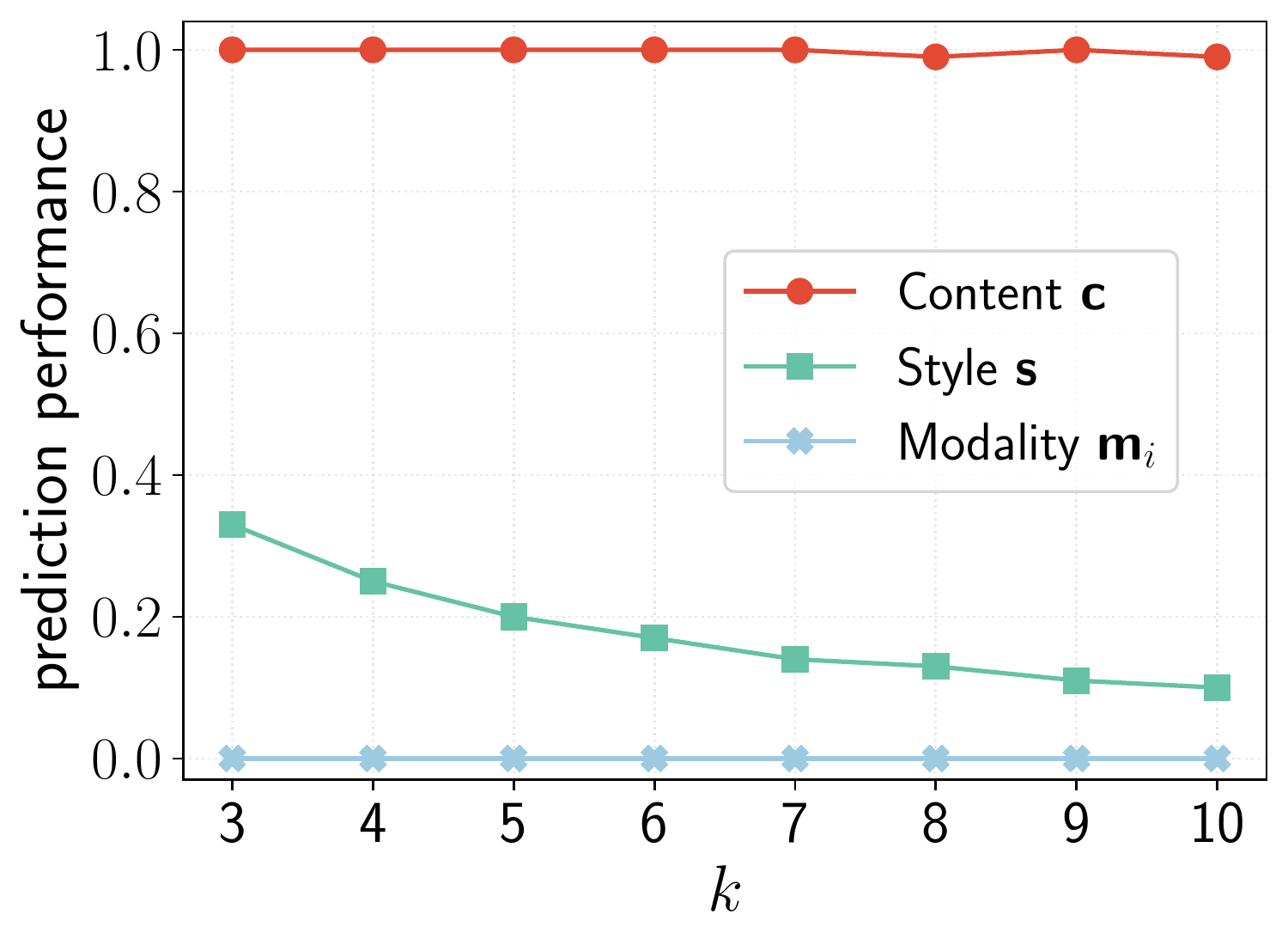}
        \caption{Only $\s$ is discrete}
        \label{subfig:stylediscrete}
    \end{subfigure}%
    \quad
    \begin{subfigure}[t]{0.3\textwidth}
        \includegraphics[width=1.0\textwidth]{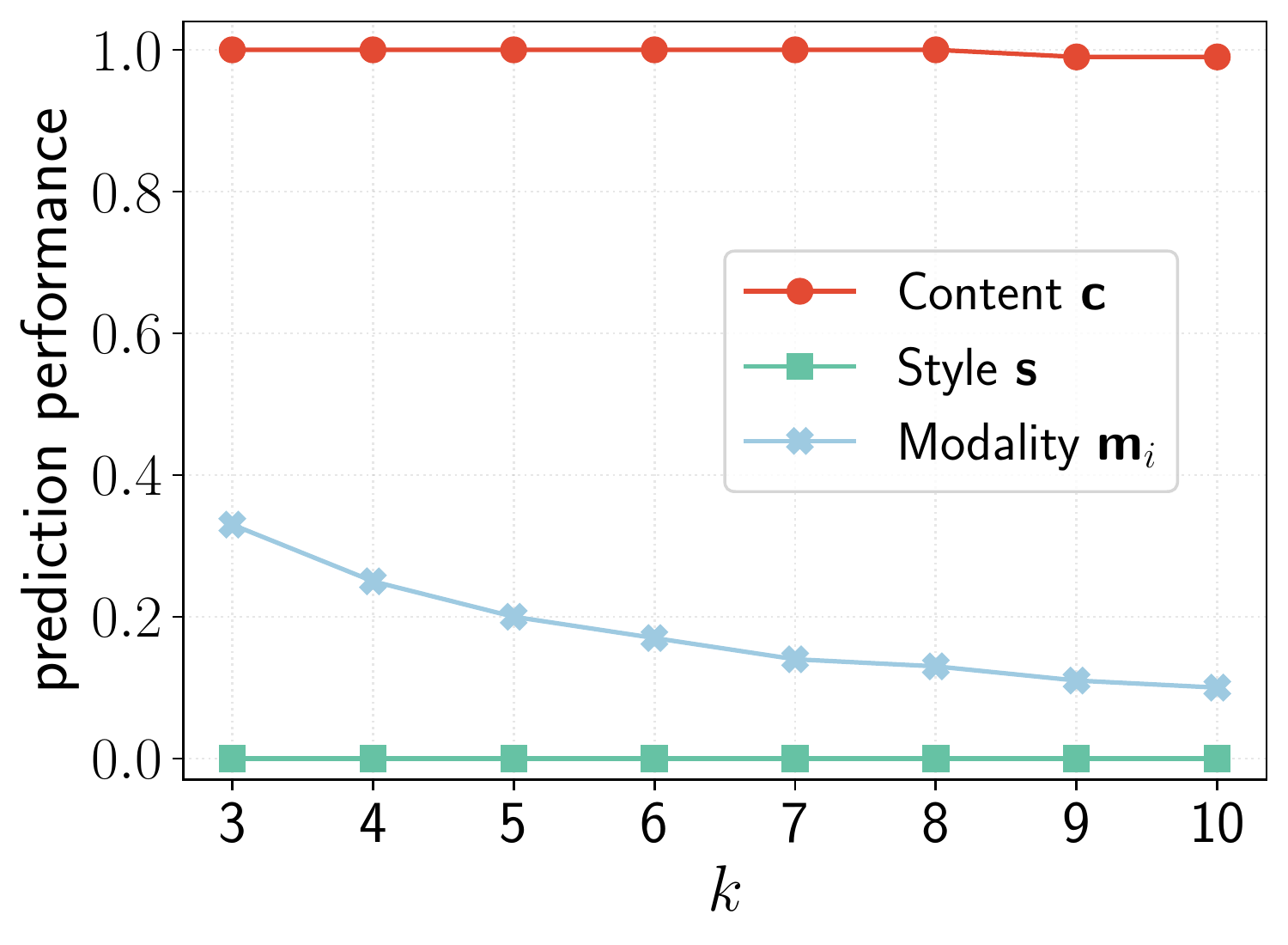}
        \caption{Only $\m_1, \m_2$ are discrete}
        \label{subfig:msdiscrete}
    \end{subfigure}%
    \quad
    \begin{subfigure}[t]{0.3\textwidth}
        \centering
        \includegraphics[width=1.0\textwidth]{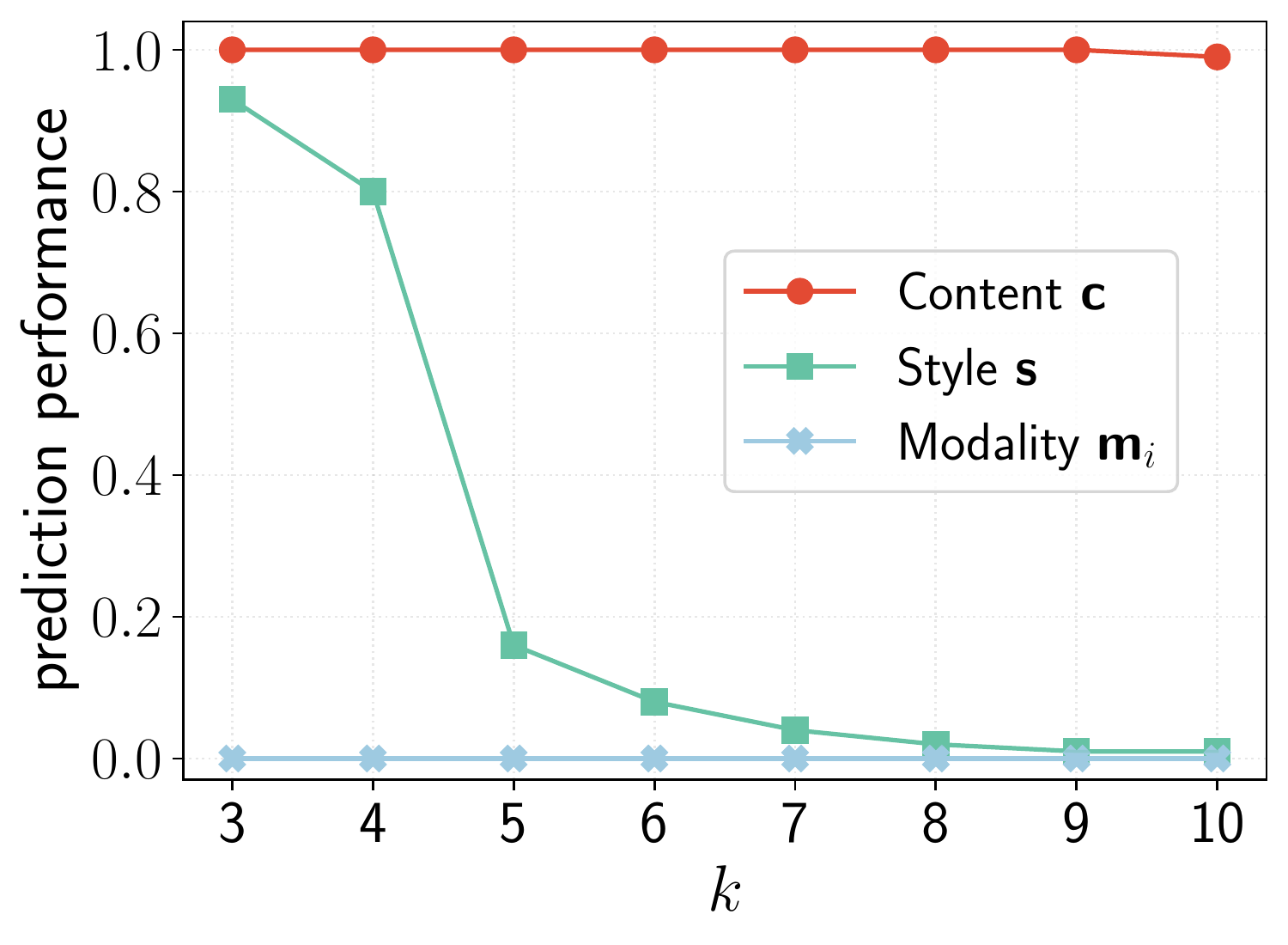}
        \caption{Only $\c$ is discrete}
        \label{subfig:contentdiscrete}
    \end{subfigure}
    \hfill
    \caption{%
      Numerical simulations with discrete latent factors. The results
      show three settings in each of which one group of latent variables is
      discrete while the remaining groups are continuous. Continuous variables
      are normally distributed, whereas discrete variables are sampled from a
      multinomial distribution with $k$ distinct classes.  We measure the
      prediction performance with a nonlinear model in terms of the $R^2$
      coefficient of determination for continuous factors and classification
      accuracy for discrete factors respectively. Each point denotes the
      average across three seeds and error bars show the standard deviation,
      which is relatively small. 
    }
\label{fig:discrete} 
\end{figure}

\paragraph{Dimensionality ablations for the numerical simulation}
To test the effect of latent dimensionality on identifiability,
\Cref{fig:numerical_dimensionality_ablations} presents dimensionality ablations
where we keep the number of content dimensions fixed and only vary the number
of style or modality-specific dimensions, $n_s$ and $n_m$ respectively.
\Cref{subfig:numerical_mdim_ablation,subfig:numerical_sdim_ablation} confirm
that block-identifiability of content still holds when we significantly
increase the number of style or modality-specific dimensions, as the
representation consistently encodes only content and no style or
modality-specific information. In
\Cref{subfig:numerical_mdim_learning_curves,subfig:numerical_sdim_learning_curves},
we can observe that the training loss decreases more slowly when we increase
the dimensionality of $n_c$ and $n_s$ respectively, which provides an intuition
that the sample complexity might increase with the number of style and
modality-specific dimensions.

\begin{figure}[p]
\vspace{-2.0em}
    \centering
    \begin{subfigure}[t]{0.45\textwidth}
        \centering
        \includegraphics[width=1.0\textwidth]{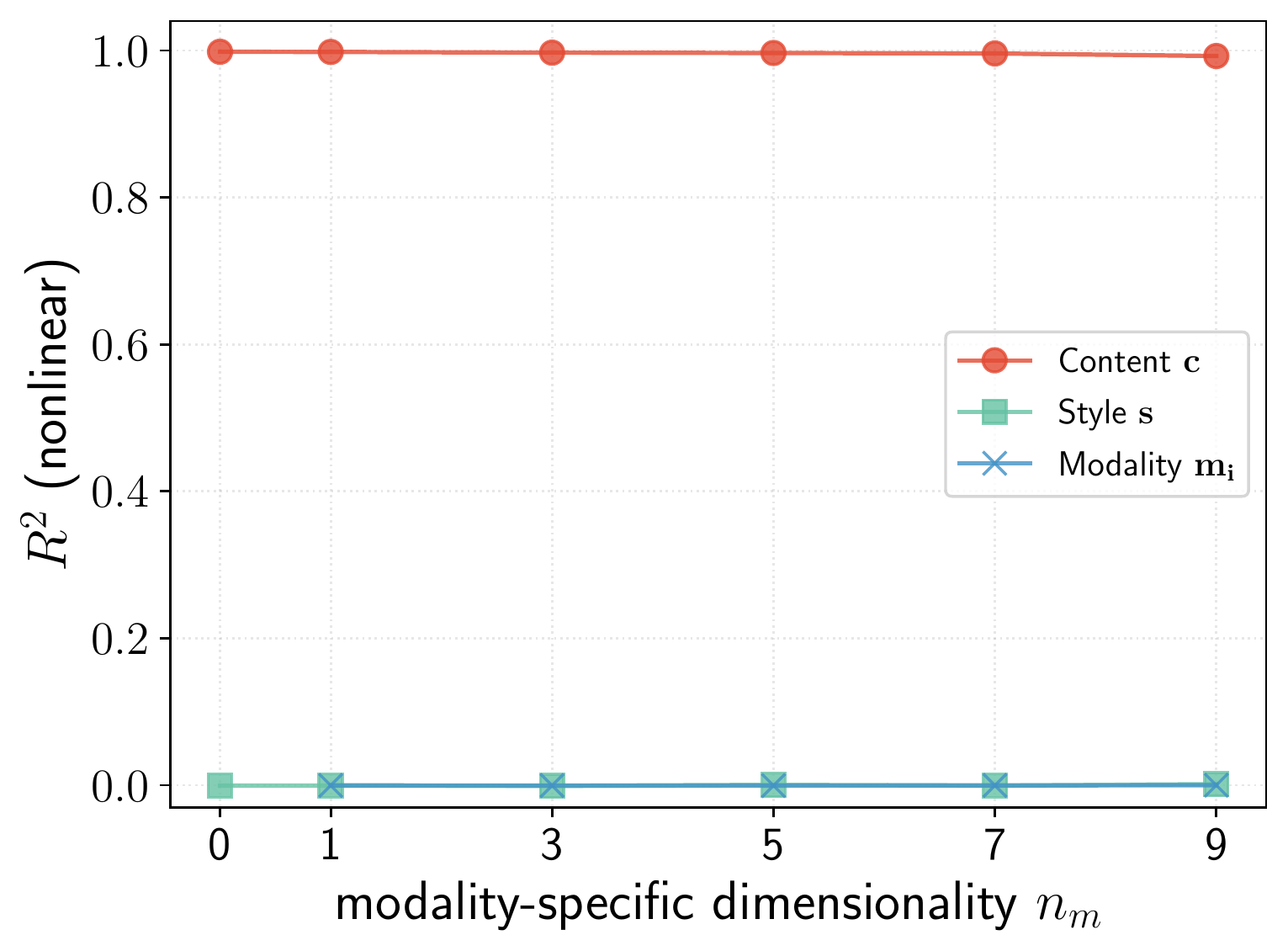}
        \caption{Prediction performance as a function of $n_m$}
        \vspace{0.75em}
        \label{subfig:numerical_mdim_ablation}
    \end{subfigure}
    \quad
    \begin{subfigure}[t]{0.45\textwidth}
        \centering
        \includegraphics[width=1.0\textwidth]{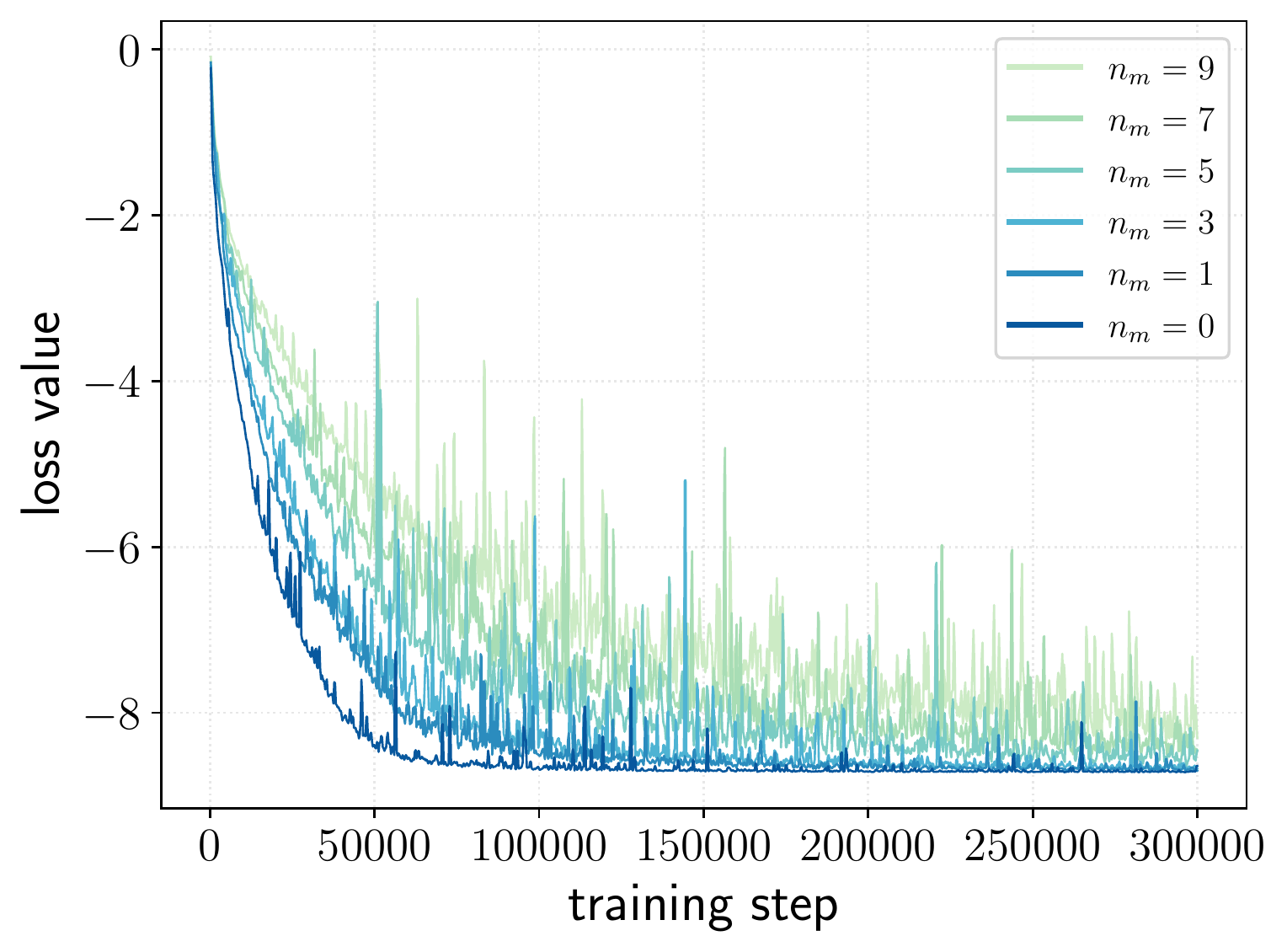}
        \caption{Learning curves for different $n_m$}
        \label{subfig:numerical_mdim_learning_curves}
    \end{subfigure}
    \begin{subfigure}[t]{0.45\textwidth}
        \centering
        \includegraphics[width=1.0\textwidth]{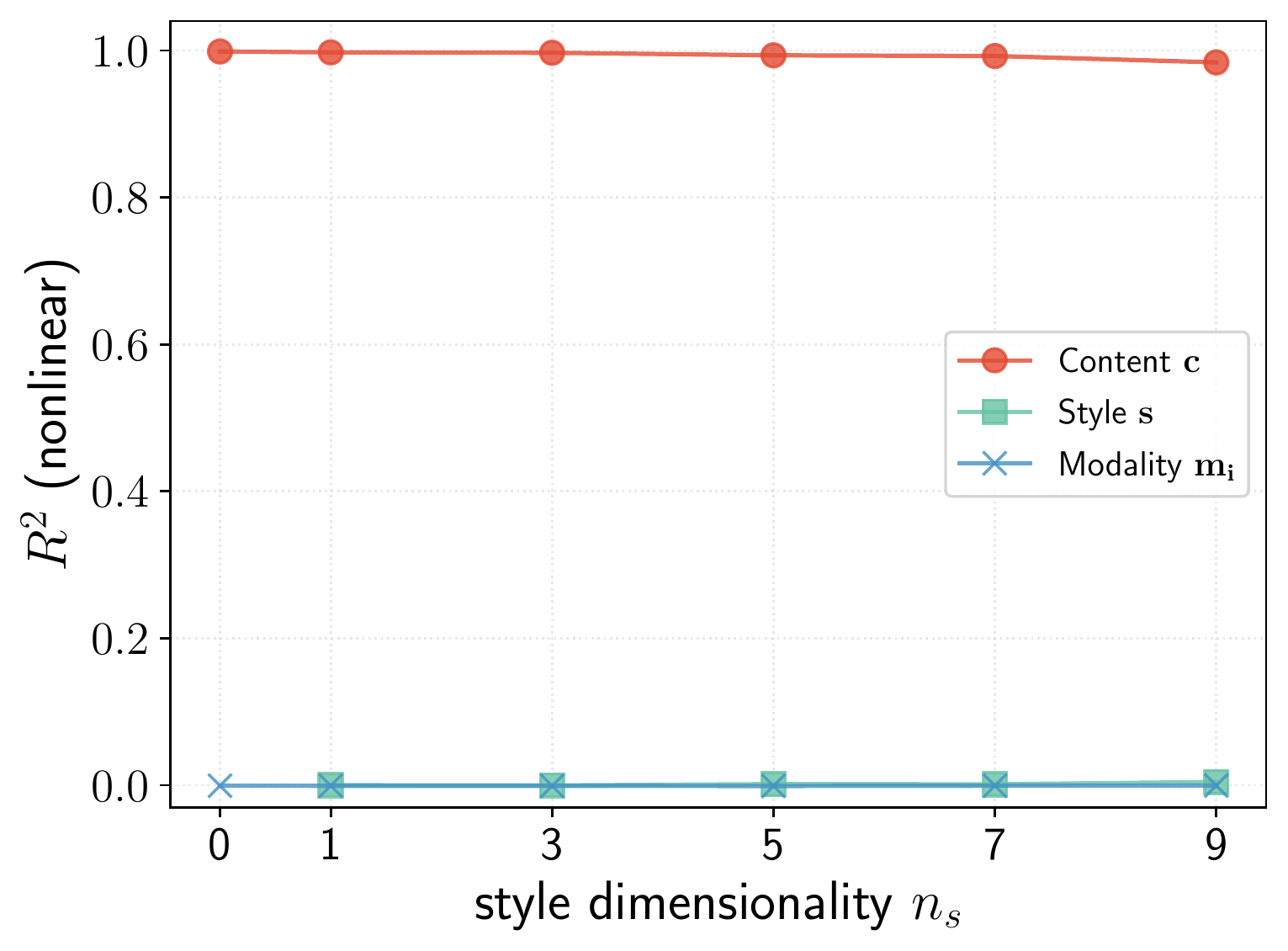}
        \caption{Prediction performance as a function of $n_s$}
        \label{subfig:numerical_sdim_ablation}
    \end{subfigure}
    \quad
    \begin{subfigure}[t]{0.45\textwidth}
        \centering
        \includegraphics[width=1.0\textwidth]{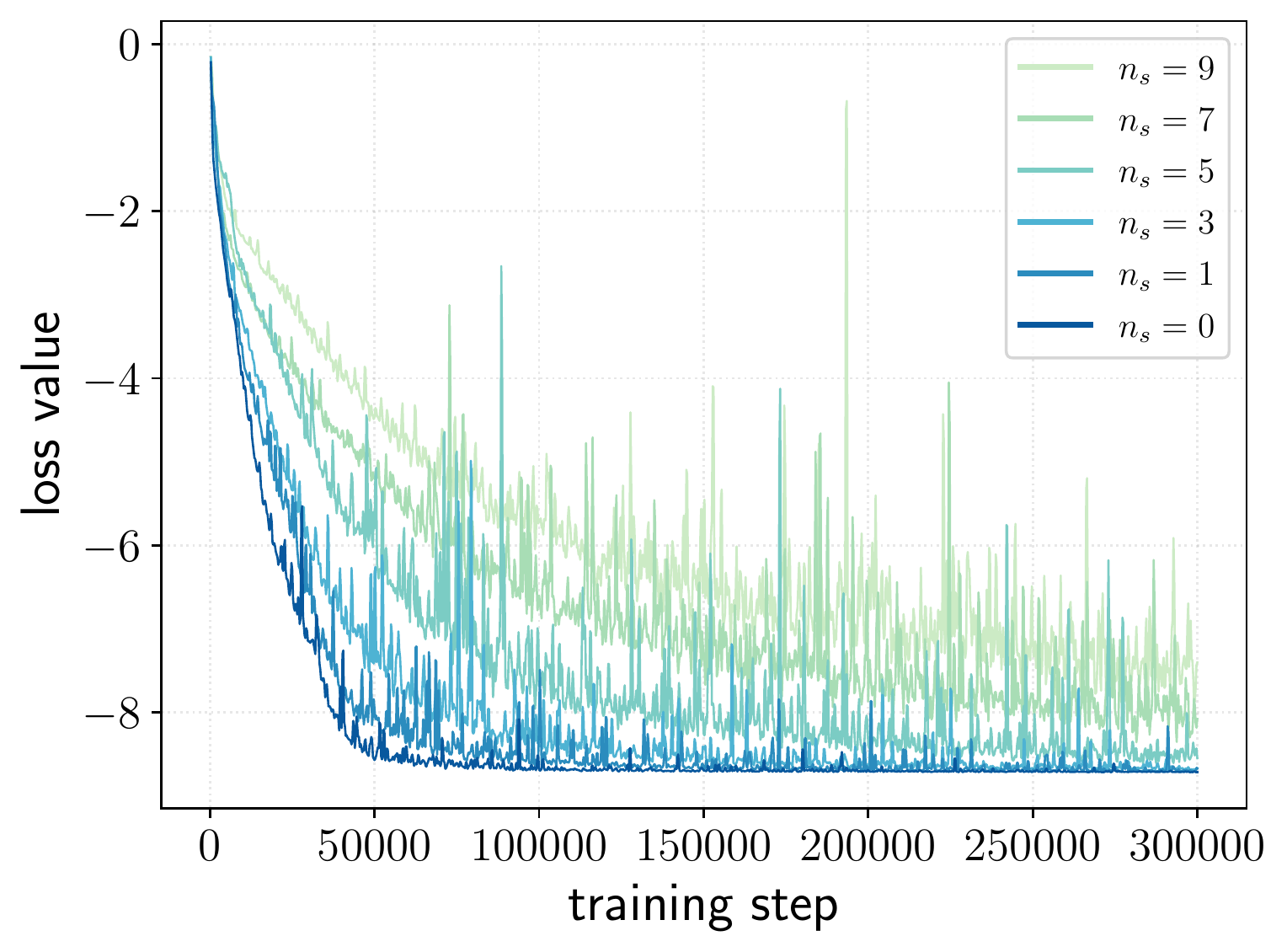}
        \caption{Learning curves for different $n_s$}
        \label{subfig:numerical_sdim_learning_curves}
    \end{subfigure}
    \caption{%
      Dimensionality ablation for the numerical simulation. We consider
      the multimodal setting with mutually independent factors and test the
      effect of latent dimensionality on identifiability by keeping the number
      of content dimensions fixed and only varying the number of style or
      modality-specific dimensions ($n_s$ and $n_m$ respectively). In
      \Cref{subfig:numerical_mdim_ablation,subfig:numerical_sdim_ablation} we
      measure the nonlinear prediction performance in terms of the $R^2$
      coefficient of determination of a nonlinear regression model that
      predicts the respective ground truth factor ($\c$, $\s$, or $\m_i$) from
      the learned representation. In
      \Cref{subfig:numerical_mdim_learning_curves,subfig:numerical_sdim_learning_curves},
      we plot the learning curves (i.e., the training loss) of the respective
      models to compare how fast they converge.
    }
\label{fig:numerical_dimensionality_ablations}
\vspace{2.5em}
    \centering
    \begin{subfigure}[t]{0.45\textwidth}
        \includegraphics[width=1.0\textwidth]{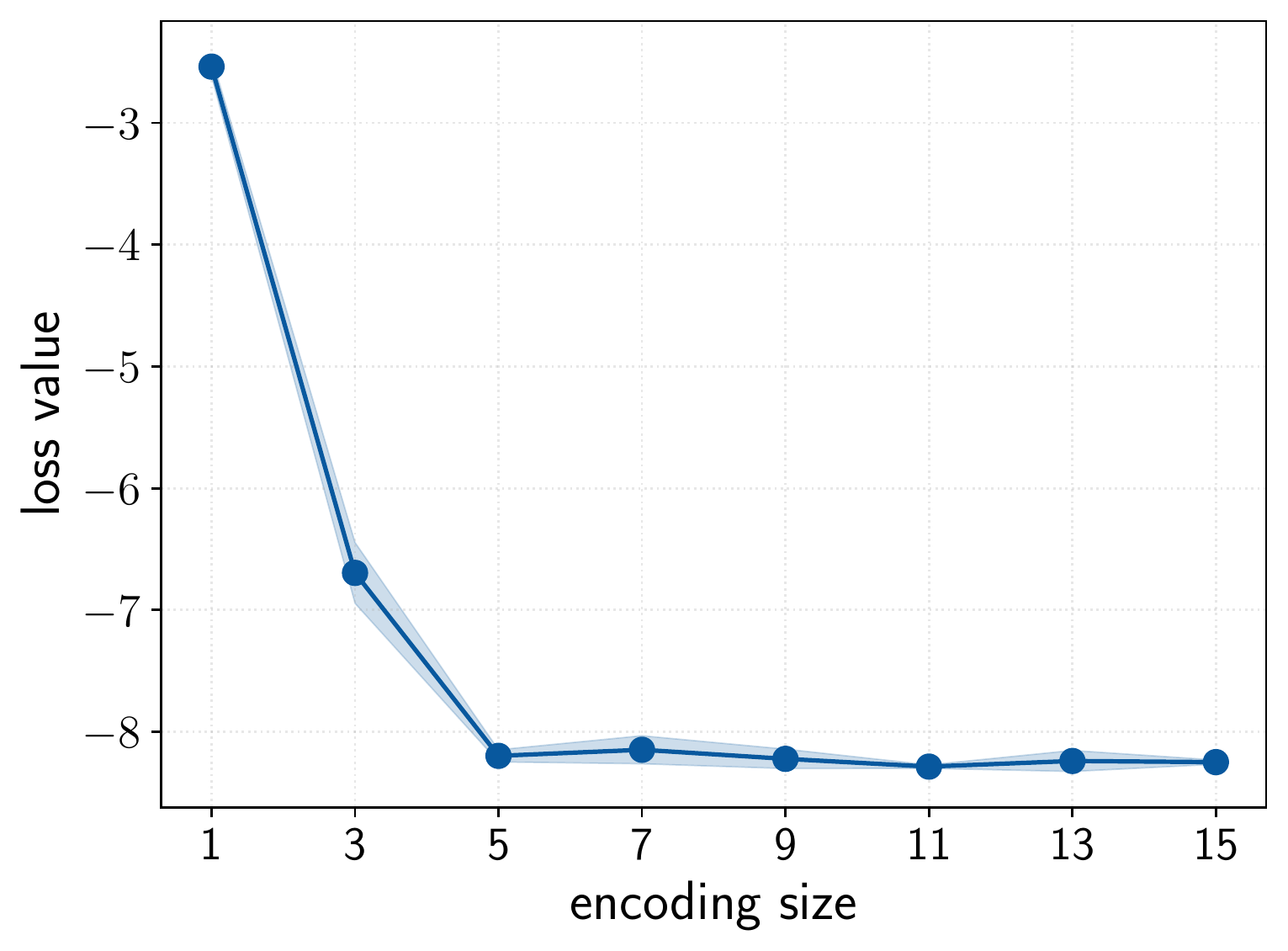}
        \caption{Validation loss for the numerical simulation}
        \label{subfig:model_selection_numerical}
    \end{subfigure}%
    \quad
    \begin{subfigure}[t]{0.45\textwidth}
        \centering
        \includegraphics[width=1.0\textwidth]{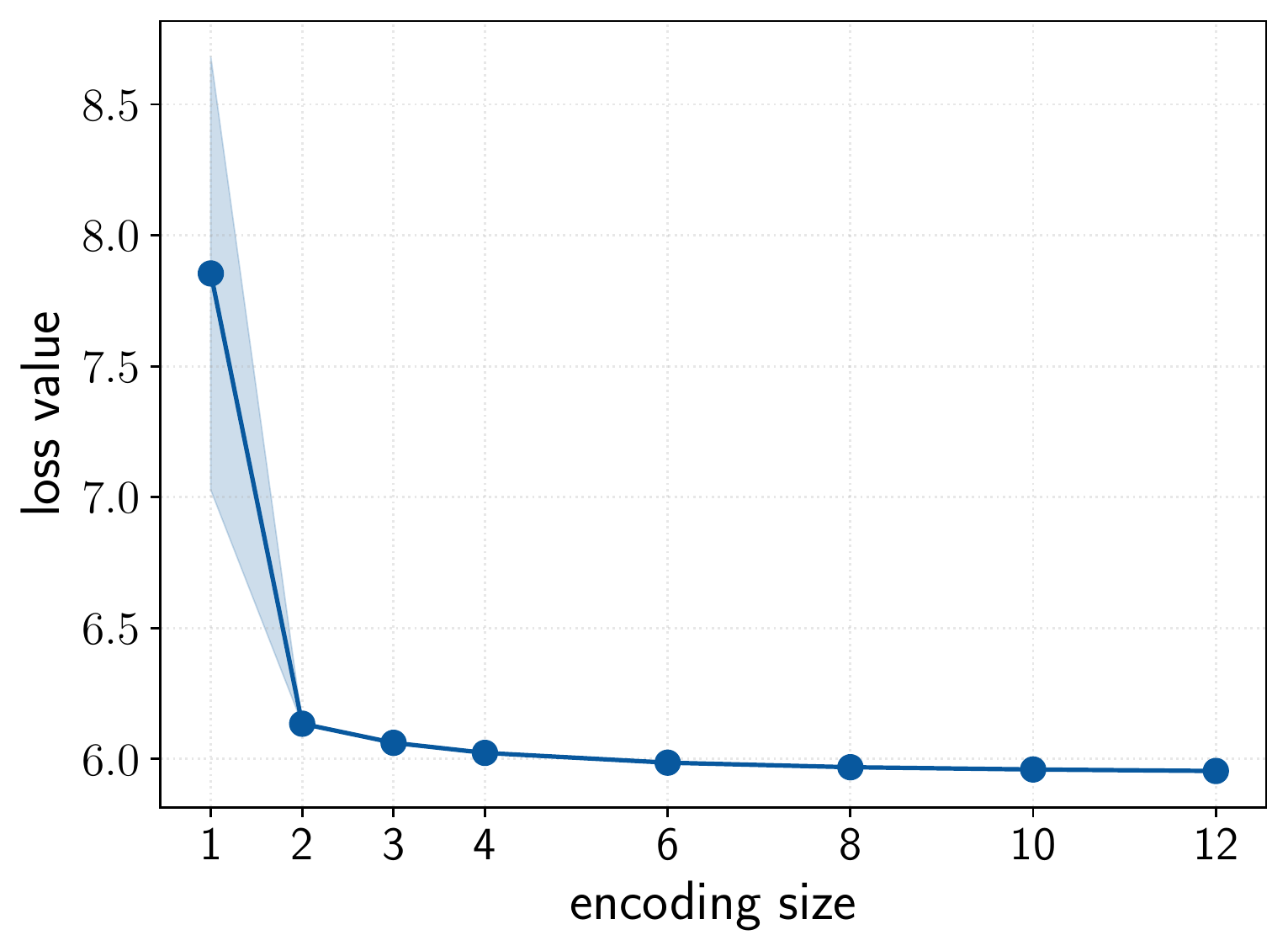}
        \caption{Validation loss for the image/text experiment}
        \label{subfig:model_selection_imagetext}
    \end{subfigure}
    \caption{%
      An attempt at estimating of the number of content factors using the
      validation loss. The validation loss corresponds to the value of the
      $\SymInfoNCE$ objective computed on a holdout dataset. Since we are
      interested in estimating the true number of content factors to select the
      encoding size appropriately, we plot the validation loss as a function of
      the encoding size. We show the validation loss for the numerical simulation
      with independent factors (\Cref{subfig:model_selection_numerical}) and for
      the image/text experiment (\Cref{subfig:model_selection_imagetext})
      respectively.
    }
\label{fig:model_selection}
\end{figure}

\paragraph{Estimating the number of content factors}
The estimation of the number of content factors is an important puzzle piece,
since \Cref{th:main} assumes that the number of content factors is known or
that it can be estimated.  In practice, the number of content factors can be
viewed as a single hyperparameter \citep[e.g.,][]{Locatello2020} that can be
tuned with respect to a suitable model selection metric. For instance, one
could use the validation loss for model selection, which would be convenient
since the validation loss only requires a holdout dataset and no additional
supervision. In \Cref{fig:model_selection}, we plot the validation loss
(averaged over 2,000 validation samples) as a function of the encoding size for
both experiments used in our paper. Results for the numerical simulation are
shown in \Cref{subfig:model_selection_numerical} and for the image/text
experiment in \Cref{subfig:model_selection_imagetext}. For both datasets, we
observe that the validation loss increases most significantly in the range
around the true number of content factors. For the numerical simulation, the
results look promising as they show a clear ``elbow'' \citep{James2013} at the
correct value of 5, which corresponds to the true number of content factors.
The results are less clear for the image/text experiment, where the elbow
method might suggest the range of 2-4 content factors, while the true value is
3. While these initial results look promising, we believe that more work is
required to investigate the estimation of the number of content factors and the
design of suitable heuristics, which are interesting directions for future
research.

\newpage
\clearpage

\paragraph{Evaluation with test-time interventions}
Previously, we observed that style can be predicted to some degree when there
are causal dependencies from content to style (\Cref{tab:numerical_m1m2}),
which can be attributed to style information being partially predictable from
the encoded content information in the causal setup. To verify that the
encoders only depend on content information (i.e., that content is
block-identified), we assess the trained models using a novel, more rigorous
empirical evaluation for the numerical simulation. We test the effect of
\emph{interventions} $\c \to \c'$, which perturb the content information
\emph{at test time} via batch-wise permutations of content, before generating
$\x_1' = \f_1(\c',\s, \m_1)$ and $\x_2' = \f_1(\c', \ts,\m_1)$. Hence, we break
the causal dependence between content and style (see illustration in
\Cref{tab:numerical_ablation_permcontent}), which allows us to better assess
whether the trained encoders depend on content or style information.
Specifically, we train the encoders for 3,000,000 iterations to ensure
convergence and then train nonlinear regression models to predict both the
original and the intervened content variables from the learned representations.
\Cref{tab:numerical_ablation_permcontent} presents our results using the
interventional setup, showing that in most cases only content information can
be recovered. We observe an exception (underlined values) in the two cases with
statistical dependencies, where some style information can be recovered, which
is expected because statistical dependencies reduce the effective
dimensionality of content \citep[cp.][]{Kuegelgen2021}. Analogously, in the
case of statistical and causal dependencies, some of the original content
information can be recovered via the encoded style information. In summary, the
evaluation with interventions provides a more rigorous assessment of
block-identifiability in the causal setup, showing that neither style nor
modality-specific information can be recovered when the encoding size matches
the true number of content dimensions.

\begin{figure}[t]
  \centering
  \scalebox{.79}{%
    \begin{tikzpicture}
    \node[obs] (X1) {$\x_1'$};%
    \node[latent,above=of X1,xshift=-1.5cm] (M1) {$\m_1$};%
    \node[latent,above=of X1,xshift=-0.5cm] (S1) {$\s$}; %
    \node[latent,above=of X1,xshift=0.5cm] (C) {$\c$}; %
    \node[latent,above=of X1,xshift=1.5cm] (C1) {$\c'$}; %
    \node[obs, xshift=2cm] (X2) {$\x_2'$};%
    \node[latent,above=of X2,xshift=0.5cm] (S2) {$\ts$}; %
    \node[latent,above=of X2,xshift=1.5cm] (M2) {$\m_2$};%
    \node[above left=0.01cm and 0.01cm of C1, outer sep = 0pt, inner sep = 0pt, text = red,rotate=-46,scale=1.2] {$\boldsymbol{\leadsto}$};
    \edge{M1,C1,S1}{X1}%
    \edge{M2,C1,S2}{X2}%
    \edge{C}{S1}%
    \edge{C}{C1}%
    \edge[bend left=50]{S1}{S2}%
  \end{tikzpicture}
  } %
  \;
  \resizebox{!}{2.4cm}{%
    \small
    \begin{tabular}[b]{ccccccc}
      \toprule
      \multicolumn{3}{c}{\textbf{Generative process}} & \multicolumn{3}{c}{$\bm{R^2}$ \textbf{(nonlinear)}}  \\
      \cmidrule(r){1-3}\cmidrule(r){4-7}
      \textbf{p(chg.)} & \textbf{Stat.} & \textbf{Cau.} & \textbf{Content $\c$} & \textbf{Content $\c'$} & \textbf{Style $\s$} & \textbf{Modality $\m_i$}\\
      \midrule
     1.0 & \xmark & \xmark  & $0.00 \pm 0.00$ & $\textbf{1.00} \pm 0.00$ &$0.00 \pm 0.00$ &$0.00 \pm 0.00$ \\
     0.75 & \xmark & \xmark & $0.00 \pm 0.00$ & $\textbf{1.00} \pm 0.00$ &$0.00 \pm 0.00$ &$0.00 \pm 0.00$ \\
     0.75 & \cmark & \xmark & $0.00 \pm 0.00$ & $\textbf{1.00} \pm 0.00$ &$\underline{0.50} \pm 0.19$ &$0.00 \pm 0.00$ \\
     0.75 & \xmark & \cmark & $0.01 \pm 0.00$ & $\textbf{0.98} \pm 0.00$ &$0.03 \pm 0.01$ &$0.00 \pm 0.00$ \\
     0.75 & \cmark & \cmark & $\underline{0.28} \pm 0.14$ & $\textbf{0.91} \pm 0.03$ &$\underline{0.39} \pm 0.20$ &$0.00 \pm 0.00$ \\
      \bottomrule
    \end{tabular}
  }  %
  \caption{%
    Evaluation with test-time interventions. We use the interventional setup
    that is illustrated on the left, i.e., perturbed samples $\x_1', \x_2'$
    that are generated from the intervened content $\c'$, which is a copy of
    the original content $\c$ with an intervention, i.e., a batch-wise
    permutation (\textcolor{red}{$\leadsto$}) that makes $\c'$ independent of
    $\s$.  Each row presents the results of a different setup with varying
    style-change probability p(chg.) and possible statistical (Stat.) and/or
    causal (Caus.) dependencies. Each value denotes the $R^2$ coefficient of
    determination (averaged across 3 seeds) for a nonlinear regression model
    that predicts the respective ground truth factor ($\c, \c', \s$, or $\m_i$)
    from the learned representation.
  }
\label{tab:numerical_ablation_permcontent}
\end{figure}
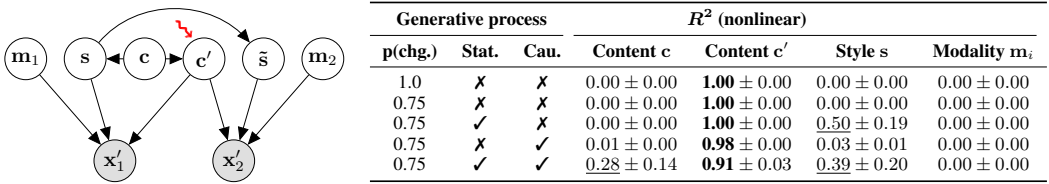 

\paragraph{High-dimensional image pairs with continuous latents}
\begin{wrapfigure}{r}{0.28\textwidth} 
\vspace{-4pt}
  \begin{center}
    \includegraphics[width=0.28\textwidth]{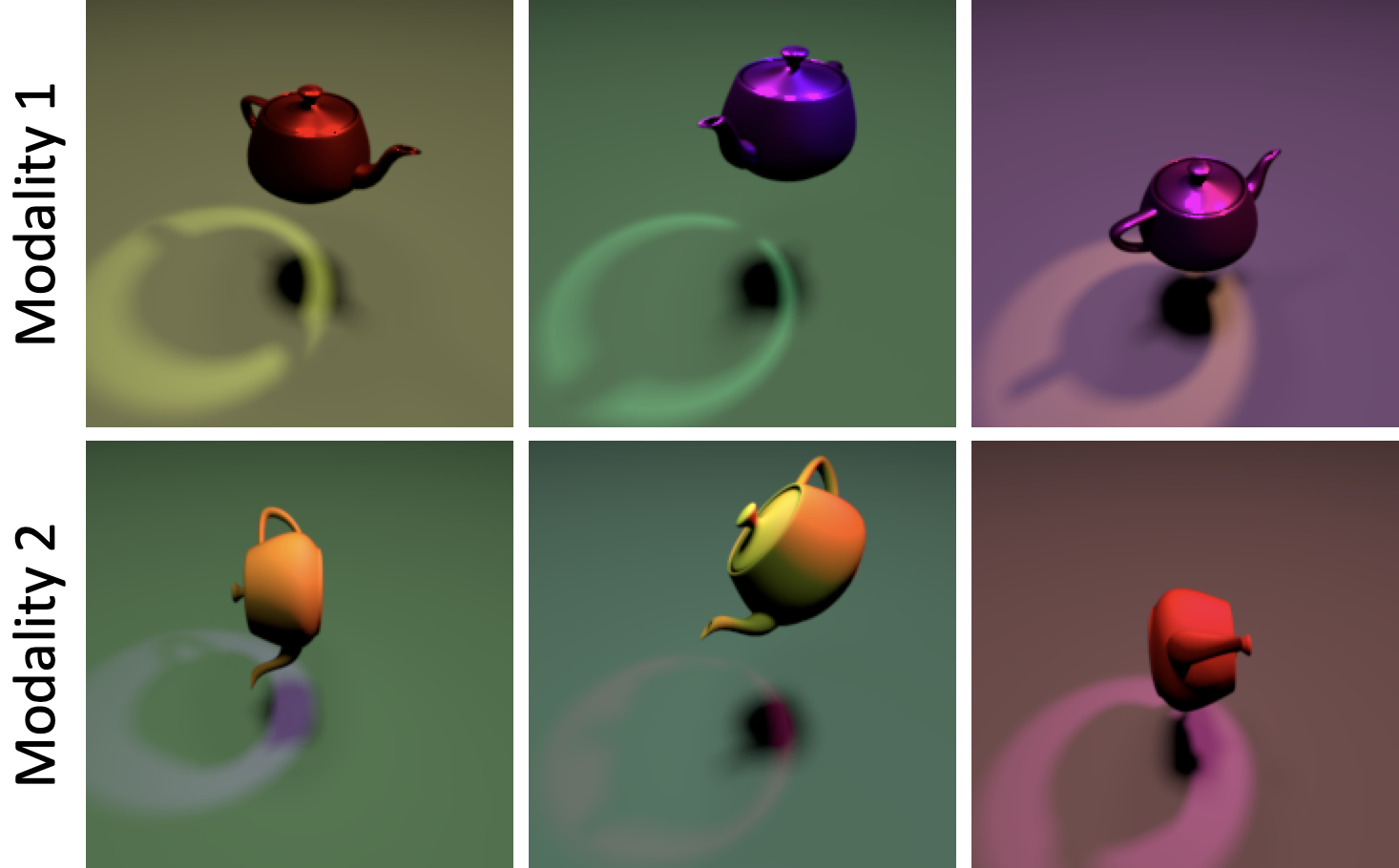}
    \vspace{-15pt}
    \caption{Examples of high-dimensional image~pairs.}
    \label{fig:m3DIdent0}
  \end{center}
\vspace{-13pt}
\end{wrapfigure} 
To bridge the gap between continuous and discrete data, we provide an
additional experiment that offers a realistic setup but uses only continuous
latent variables to satisfy the assumptions of \Cref{th:main}.  Previously, in
\Cref{subsec:imagetext_experiment}, we used a complex multimodal dataset of
image/text pairs, which were generated from a combination of continuous and
discrete latent factors. Now, we consider a different dataset that consists of
\emph{pairs high-dimensional images} generated from a set of continuous
latents, which is more in line with our theoretical assumptions. Note that
datasets with pairs of images are common in practice, for example, in medical
imaging where patients are assessed using multiple views (e.g., images from
different angles) or multiple modalities (e.g., as in PET-CT imaging).  To
generate the data, we adapt the code from 3DIdent \citep{Zimmermann2021} to
render pairs of images, for which the object position is always shared (i.e.,
content), the object-, spotlight- and background-color is stochastically shared
(i.e., style), and modality-specific factors are object rotation for one
modality and spotlight position for the other. Additionally, we render the
objects using different textures to simulate a modality-specific mixing
process.  Samples of image pairs are shown in \Cref{fig:m3DIdent0} and further
details about the dataset can be found in
\Cref{sec:app-details-to-experimental-setting}.  We train the encoders with the
InfoNCE objective for 60,000 iterations using the same architectures and
hyperparameters as for \emph{Multimodal3DIdent} (\Cref{tab:imgtxt_details}),
and again evaluate the $R^2$ coefficient of determination using a kernel ridge
regression that predicts the respective ground truth factor from the learned
representations.

\Cref{fig:multiview_results} present our results for the dataset of image
pairs, showing the prediction performance as a function of the encoding size
for the setting with causal dependencies (\Cref{fig:multiview_causal}) and the
setting with mutually independent latent variables
(\Cref{fig:multiview_noncausal}) respectively.  In both settings, content
information (i.e., object position) is recovered when sufficient encoding
capacity is available. Style and modality-specific information, on the other
hand, are discarded independent of the encoding size. In
\Cref{fig:multiview_causal} we observe the recovery of some style information,
which is expected because style can be predicted to some degree from the
encoded content information when there is a causal dependence of style on
content. Overall, these findings lend further support to our theoretical result
from \Cref{th:main}, as we investigate a realistic setting with only continuous
latent factors, which is more in line with our assumptions. Notably, the
results appear more consistent with our theory, e.g., showing that less style
and modality-specific information is encoded, compared to our results for the
image/text experiment, where we used a combination of continuous and discrete
latent factors.

\begin{figure*}[t]
    \centering
    \begin{subfigure}[t]{0.49\textwidth}
        \includegraphics[width=1.0\textwidth]{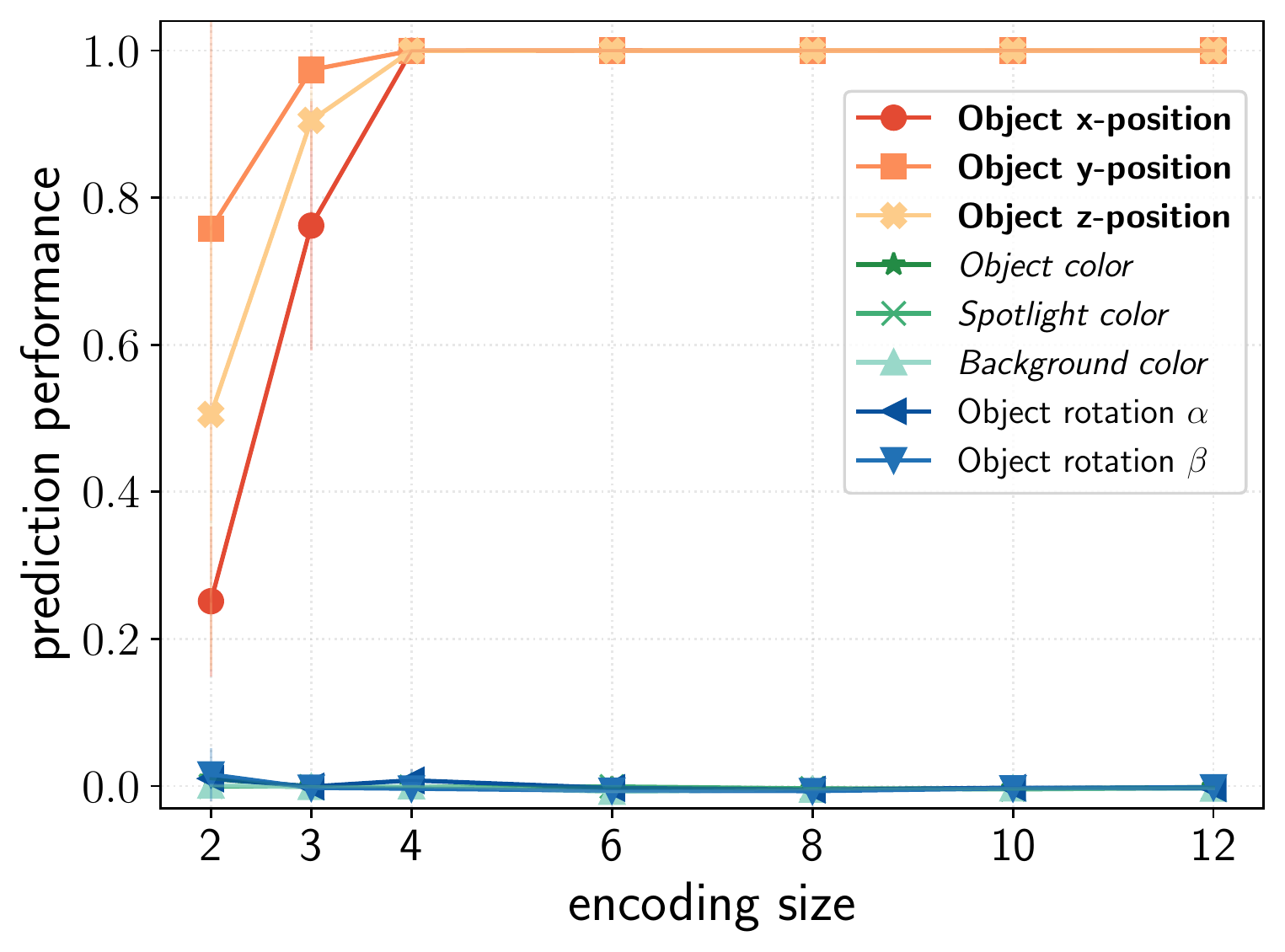}
        \caption{With mutually independent factors}
        \label{fig:multiview_noncausal}
    \end{subfigure}%
    \hfill
    \begin{subfigure}[t]{0.49\textwidth}
        \centering
        \includegraphics[width=1.0\textwidth]{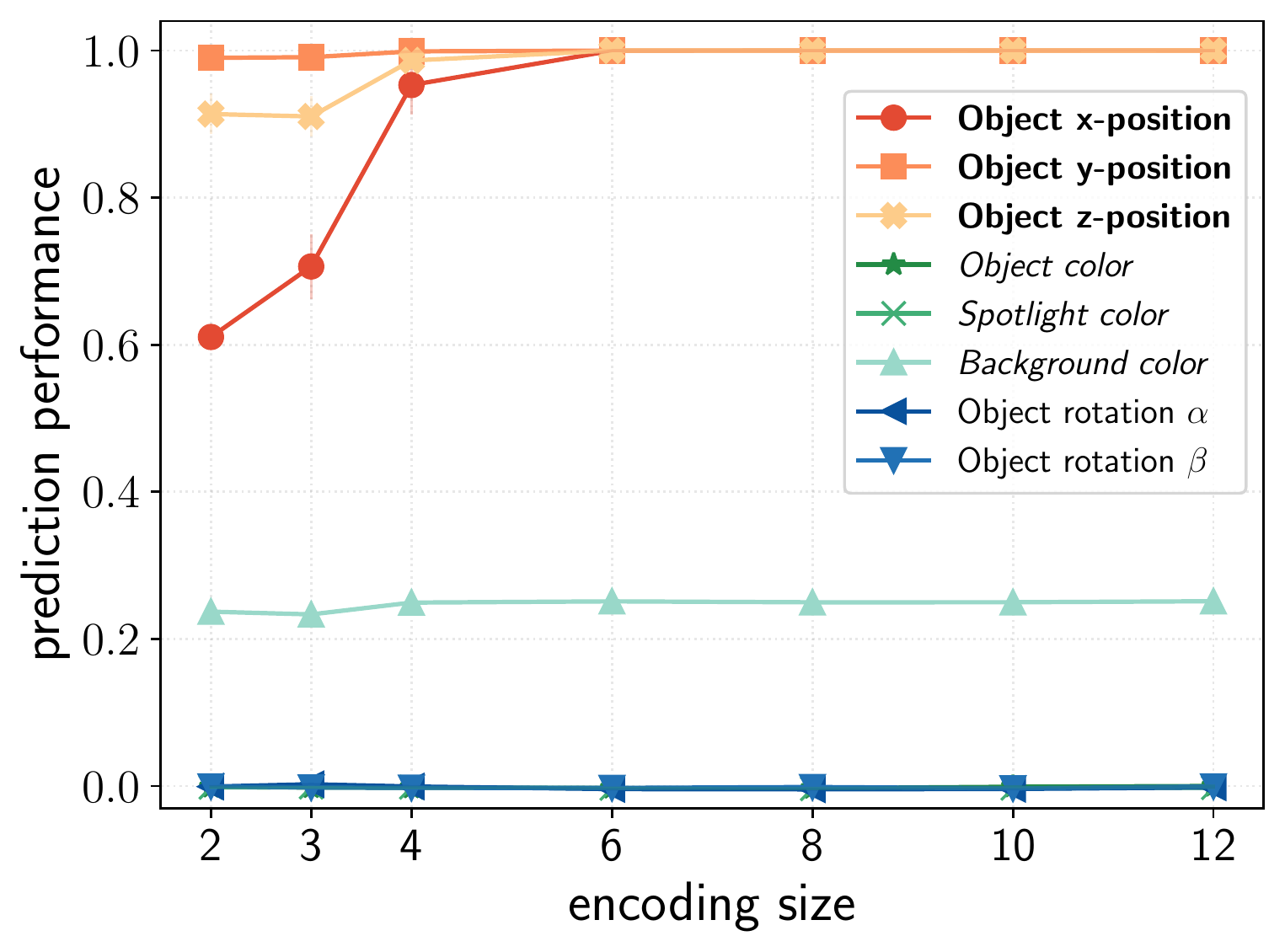}
        \caption{With causal dependencies}
        \label{fig:multiview_causal}
    \end{subfigure}
    \caption{%
      Result with pairs of high-dimensional images. As a function of the encoding
      size of the model, we assess the nonlinear prediction of ground truth
      factors to quantify how well the learned representation encodes the
      respective factors. Content factors are denoted in bold, style factors in
      italic, and modality-specific factors in regular font. Each point denotes
      the average $R^2$ score across three seeds and bands show one standard
      deviation. 
    }
\label{fig:multiview_results}
\end{figure*}

\paragraph{Multimodal3DIdent with mutually independent factors}
For the results of the image/text experiment in the main text
(\Cref{subsec:imagetext_experiment}) we used the \emph{Multimodal3DIdent}
dataset, which we designed such that object color is causally dependent on the
x-position of the object to impose a causal dependence of style on content. In
\Cref{fig:imagetext_noncausal}, we provide a similar analysis using a version
of the dataset \emph{without} the causal dependence, i.e., with mutually
independent factors. For both modalities, we observe that object color is only
encoded when the encoding size is larger than four, i.e., when there is excess
capacity beyond the capacity needed to encode all content factors. Hence, these
results corroborate that contrastive learning can block-identify content
factors in a complex multimodal setting with heterogeneous image/text pairs.

\begin{figure*}[t]
\vspace{2em}
   \centering
    \begin{subfigure}[t]{0.49\textwidth}
        \includegraphics[width=1.0\textwidth]{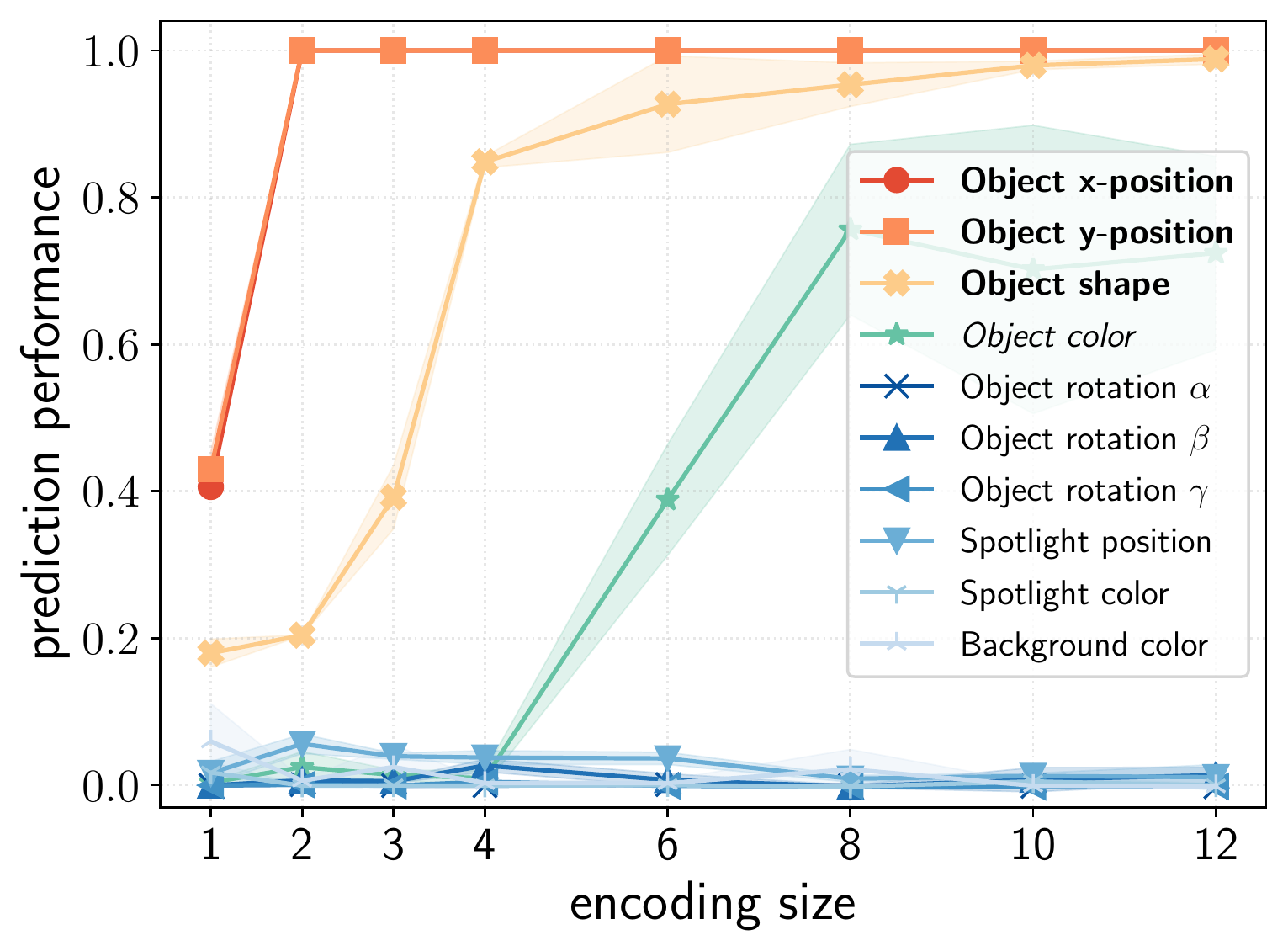}
        \caption{Prediction of image factors}
        \label{subfig:imagetext_results_image_noncausal}
    \end{subfigure}%
    \hfill
    \begin{subfigure}[t]{0.49\textwidth}
        \centering
        \includegraphics[width=1.0\textwidth]{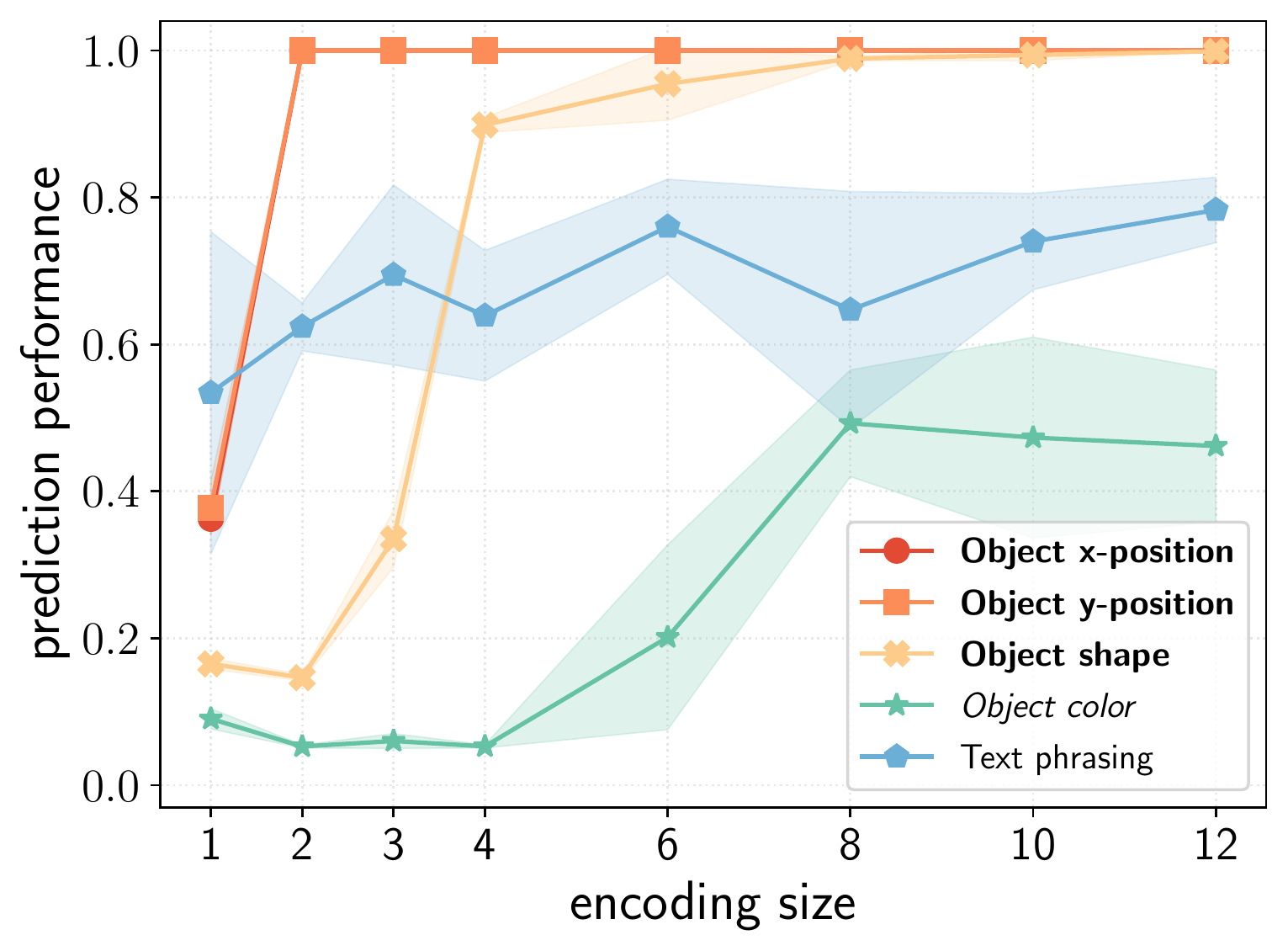}
        \caption{Prediction of text factors}
        \label{subfig:imagetext_results_text_noncausal}
    \end{subfigure}
    \caption{%
      Result on \emph{Multimodal3DIdent} with mutually independent factors. As
      a function of the encoding size of the model, we assess the nonlinear
      prediction of ground truth image factors (left subplot) and text factors
      (right subplot) to quantify how well the learned representation encodes
      the respective factors. Content factors are denoted in bold and style
      factors in italic. Along the x-axis, we vary the encoding size, i.e., the
      output dimensionality of the model. We measure the prediction performance
      in terms of the $R^2$ coefficient of determination for continuous factors
      and classification accuracy for discrete factors respectively. Each point
      denotes the average across three seeds and bands show one standard
      deviation. 
    }
\label{fig:imagetext_noncausal}
\end{figure*}

\end{document}